%% file: main.tex
\documentclass{l4dc2024}

\title[DC4L: Distribution Shift Recovery via Data-Driven Control for Deep Learning Models]{DC4L: Distribution Shift Recovery via Data-Driven Control for Deep Learning Models}

\usepackage{times}
\usepackage{bbm}
\usepackage{amsfonts}
\usepackage{amssymb}

\coltauthor{%
 \Name{Vivian Lin} \Email{vilin@seas.upenn.edu}\\
 \Name{Kuk Jin Jang} \Email{jangkj@seas.upenn.edu}\\
 \Name{Souradeep Dutta} \Email{duttaso@seas.upenn.edu}\\
 \Name{Michele Caprio} \Email{caprio@seas.upenn.edu}\\
 \Name{Oleg Sokolsky} \Email{sokolsky@seas.upenn.edu}\\
 \Name{Insup Lee} \Email{lee@seas.upenn.edu}\\
 \addr University of Pennsylvania
}

\input{macros}

\begin{document}

\maketitle

\begin{abstract}%
Deep neural networks have repeatedly been shown to be non-robust to the uncertainties of the real world, even to naturally occurring ones. A vast majority of current approaches have focused on data-augmentation methods to expand the range of perturbations that the classifier is exposed to while training. A relatively unexplored avenue that is equally promising involves sanitizing an image as a preprocessing step, depending on the nature of perturbation. In this paper, we propose to use control for learned models to recover from distribution shifts online. Specifically, our method applies a sequence of semantic-preserving transformations to bring the shifted data closer in distribution to the training set, as measured by the Wasserstein distance. Our approach is to 1) formulate the problem of distribution shift recovery as a Markov decision process, which we solve using reinforcement learning, 2) identify a minimum condition on the data for our method to be applied, which we check online using a binary classifier, and 3) employ dimensionality reduction through orthonormal projection to aid in our estimates of the Wasserstein distance. We provide theoretical evidence that orthonormal projection preserves characteristics of the data at the distributional level. We apply our distribution shift recovery approach to the ImageNet-C benchmark for distribution shifts, demonstrating an improvement in average accuracy of up to \maxICacc{} across a variety of state-of-the-art ImageNet classifiers. We further show that our method generalizes to composites of shifts from the ImageNet-C benchmark, achieving improvements in average accuracy of up to \maxICPacc{}.  Finally, we test our method on CIFAR-100-C and report improvements of up to \maxCCacc{}.
\end{abstract}


\input{sections/introduction}
\input{sections/preliminaries}
\input{sections/problem_statement}
\input{sections/training_using_RL}
\input{sections/correction_detection}
\input{sections/minimal_conditions}
\input{sections/orthogonal_projections}
\input{sections/background}
\input{sections/application}
\input{sections/conclusion}

\acks{
This research was supported in part by ARO W911NF-20-1-0080. The views and conclusions contained in this document are those of the authors and should not be interpreted as representing the official policies, either expressed or implied, of the Army Research Office or the U.S. Government.
}


\ifshowexternal
\else
\input{supplement} 
\fi

\end{document}

%% file: macros.tex
\usepackage{amsfonts}
\usepackage{xcolor}
\usepackage{mathtools}

\usepackage{soul}
\usepackage{algorithm}
\usepackage{algorithmic} 
\usepackage{multirow}
\usepackage{array}
\usepackage{xcolor, colortbl}
\usepackage{wrapfig}

\newcommand{\real}{\mathbb{R}}
\newcommand{\ours}{\texttt{SuperStAR}}
\newcommand{\maxICacc}{14.21\%}
\newcommand{\maxICPacc}{9.81\%}
\newcommand{\maxCCacc}{8.25\%}
\usepackage{enumitem}

\usepackage[format=plain]{caption}

\newif\ifshowexternal
\newcommand{\ext}{\ifshowexternal ~of~\citet{lin2023take}\fi}

\newcommand{\extref}[2]{\ifshowexternal#2\else\ref{#1}\fi} 

\showexternalfalse 

%% file: sections/introduction.tex
\section{Introduction}

\input{sections/introduction-l4dc}

In summary, our contributions towards addressing the problem of robustness to semantic preserving shifts are as follows. 1) We translate the problem of distribution shift recovery for neural networks to a Markov decision process, which we solve using reinforcement learning. 2) We identify a minimum condition of operability for our method, which we check online using a binary classifier. 3) We develop a method to efficiently compute the degree of distribution shift by projecting to a lower dimensional space. This uses results from \cite{lim2} in conjunction with the Wasserstein distance. 4) We demonstrate an application to ImageNet-C and achieve significant accuracy improvements (up to \maxICacc{} averaged across all shift severity levels) on top of standard training and data-augmentation schemes. We further show that our method generalizes beyond the ImageNet-C benchmark, yielding up to \maxICPacc{} on composite Imagenet-C shifts and up to \maxCCacc{} on CIFAR-100-C in accuracy improvements.

%% file: sections/introduction-l4dc.tex
Deep learning models are excellent at learning patterns in large high dimensional datasets. However, the brittleness of deep neural networks (DNNs) to distribution shifts is a challenging problem. \cite{hendrycks_baseline_2018} showed that, even in naturally occurring distribution shift scenarios, a classifier's performance can deteriorate substantially. Arguably, one of the most widespread uses of deep learning techniques currently is in image recognition. In this paper, we propose to use decision and control for learned models (DC4L) to improve robustness to distribution shifts, with demonstration on image classification tasks. The idea is to actively \emph{sanitize} a set of images depending on the type of the distribution shift. Intuitively, the training distribution of a classifier is viewed as its \emph{comfort zone}, where the behavior is more predictable and trustworthy. At run time, when exposed to perturbations that push the images outside of this comfort zone, a feedback policy takes control actions which can bring the images back to a more familiar space. The control actions are so chosen that the semantic meaning of the images are preserved. This ensures correctness. While approaches for DNN robustness have until now focused on data-augmentation techniques \citep{hendrycks2019augmix, noisy_mix_paper, manifold_mix, cut_mix, puzzle_mix, deep_augment}, we show that by using ideas from the data-driven control paradigm, it is possible to provide an additional level of performance boost beyond what can be offered by SOTA augmentation methods.

Our technique exploits the following observation: when distribution shift arises in the external environment due to natural causes, it persists for a certain duration of time. For instance, when a corruption in image quality occurs due to snow, this corruption does not disappear in the next image frame. This gives the system some time to \emph{adapt} and \emph{recover} from this shift by computing some semantic preserving transformations to the data. Our technique, \underline{Super}visory system for \underline{S}hif\underline{t} \underline{A}daptation and \underline{R}ecovery (\ours), applies a sequence of semantic-preserving transforms to the input data, correcting the input to align with the original training set of the classifier. We show that formulating the sequence selection problem as a Markov decision process (MDP) lends a natural solution: reinforcement learning (RL).

%% file: sections/preliminaries.tex
\section{Preliminaries} \label{sec:prelims}




We begin by assuming that the images are sampled from a measurable space $(\mathcal{X}, \mathcal{A}_{\mathcal{X}})$. 
Let $\Delta(\mathcal{X}, \mathcal{A}_\mathcal{X})$ denote the set of all probability measures on $(\mathcal{X}, \mathcal{A}_{\mathcal{X}})$. We pick a distribution $D \in \Delta(\mathcal{X}, \mathcal{A}_\mathcal{X})$ from which the current set of images are sampled. Assume that the labels belong to a measurable space $(\mathcal{Y}, \mathcal{A}_\mathcal{Y})$, and a classifier $C$ is an $\mathcal{A}_\mathcal{X} \backslash \mathcal{A}_\mathcal{Y}$ measurable map, $\mathcal{C} : \mathcal{X} \rightarrow \mathcal{Y}$. An \emph{oracle classifier} $\mathcal{C}^*$ produces the ground truth labels. Next, we define a semantic preserving transform $\mathbb{T}$.

\begin{definition}[Semantic Preserving Transform]
\label{defn:semantic_preserving}
A function $\mathbb{T} : \mathcal{X} \rightarrow \mathcal{X}$ is semantic preserving iff $\mathcal{C}^*(x) = \mathcal{C}^*(\mathbb{T}(x))$, for all $x \in \mathcal{X}$.
\end{definition}

We denote by $\mathbb{S}$ the set of all such semantic preserving transforms. In the standard empirical risk minimization (ERM) paradigm, we approximate $\mathcal{C}^*$ with some classifier $\mathcal{C}$. When measuring robustness to common corruptions, typically a corrupting transform $\mathbb{T}_c$ belongs to $\mathbb{S} $. This includes transforms like the addition of Gaussian noise, speckle noise, and alike. 
The error of a classifier $\mathcal{C}$ under the distribution $D$, is defined as
$$ err(D) := {\mathbb{E}}_{x \sim D} \left[\mathbbm{1}( \mathcal{C}^* (x) \neq \mathcal{C}(x))\right],$$
where $\mathbbm{1}$ is the standard indicator function, which evaluates to $1$ iff $\mathcal{C}^* (x) \neq \mathcal{C}(x)$.\footnote{For notational convenience, we hereafter denote both a random variable and its realization by a lowercase Latin letter.} For a robust classifier we expect $err(D)$ to be minimal for multiple choices of the distribution $D$. For instance in the case of common corruptions introduced in \citet{hendrycks2019benchmarking}, the goal is to optimize the choice of classifier $\mathcal{C}$ such that $err(D)$ is minimized, even under shift. This is typically achieved using data-augmentation schemes such as Augmix, NoisyMix, and DeepAugment. 

%% file: sections/problem_statement.tex
\vspace{-2mm}
\section{Functionality and Problem Statement} \label{sec:problem_approach}


\begin{wrapfigure}{r}{0.46\textwidth}
\vspace{-20mm}
\begin{center}
\includegraphics[width=7cm, height=3cm]{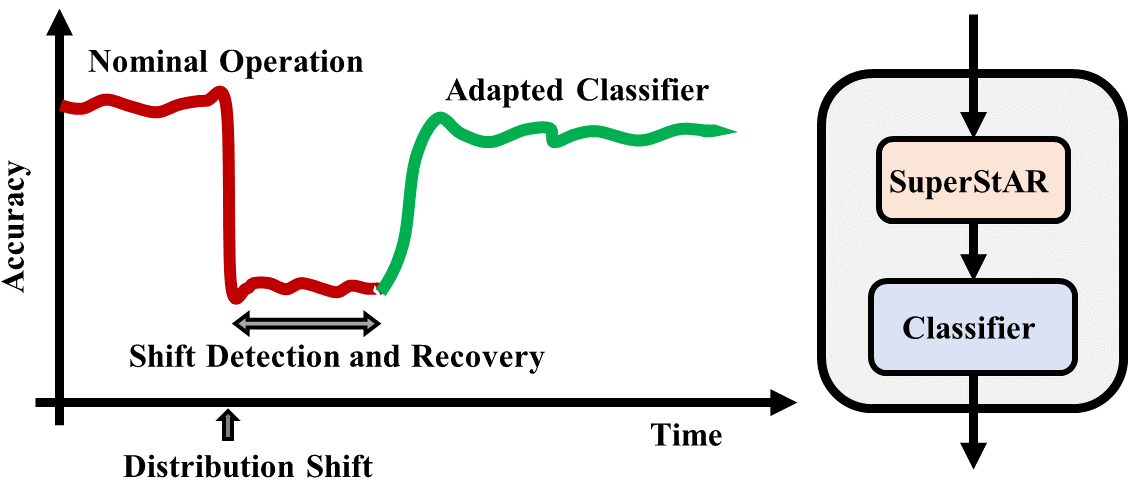}
\end{center}
\vspace{-6mm}
\caption{
    \footnotesize
    Overview of \ours{}. 
    At deployment, assume that a distribution shift causes a drop in accuracy. This is detected through changes in the Wasserstein distance between a validation set and the corrupted set. \ours{} computes a composition of transforms $\mathcal{I}_k$ to adapt to the shift and recover accuracy. This composition of \ours{} with the classifier helps it detect and adapt, boosting robustness of classification.}
\label{fig:operating_diagram}
\vspace{-2mm}
\end{wrapfigure}

At a high level, the functioning of \ours{} is akin to a supervisor for the classifier shown in Figure \ref{fig:operating_diagram}. \ours{} detects distribution shifts and computes a recovery strategy to be applied before sending an image to a classifier. Determining the appropriate recovery strategy presents an interesting challenge.

Consider a random variable $x_c = \mathbb{T} (z) $, where $z \sim D$ and $\mathbb{T}$ is a semantic preserving transform. Note that $\mathbb{T}$ is generally not invertible. However, one may possibly choose an element $\mathbb{T}^\prime$ from the set of semantic preserving transforms $\mathbb{S}$ that partially recovers the accuracy drop due to $\mathbb{T}$. It might even be effective to select a sequence of such transforms: an ordered set $\mathcal{T}:= \{ \mathbb{T}_k, \mathbb{T}_{k-1}, \dots \mathbb{T}_1 \}$ with $k \geq 1$, such that $ \mathbb{S} \ni \mathcal{I}_k(x) := \mathbb{T}_k \circ \mathbb{T}_{k-1} \circ \cdots \circ \mathbb{T}_1(x)$.\footnote{The order $\prec_\mathcal{T}$ on $\mathcal{T}$ is given by $\mathbb{T}_i \prec_\mathcal{T} \mathbb{T}_j \iff i < j$, $i,j \in\mathbb{N}$. $\mathbb{T}_{0}$ is assumed to be the identity function.} Ideally, we wish to find an algorithm which optimizes the following:

\vspace{-4mm}
\begin{equation} \textstyle
    \label{eq:optimization_problem}
    \begin{aligned}
    \mathcal{I}_k^* = \mathrm{arg min}_{\mathcal{I}_k} \text{ } R(\mathcal{I}_k) \; \text{, where} \; R(\mathcal{I}_k) = err(D_{ \mathcal{I}_k \circ \mathbb{T}}) - err(D), 
    \end{aligned}
\end{equation}
with $D_{ \mathcal{I}_k \circ \mathbb{T}}$ denoting the distribution of $\mathcal{I}_k \circ \mathbb{T}(z)$, $z\sim D$. The transform $\mathcal{I}_k$ is what effectively \emph{reverses} an image corruption due to $\mathbb{T}$. 


At deployment, when \ours{} detects a shift in distribution compared to a clean validation set, it should propose a composition of transforms $\mathcal{I}_k$ to minimize the cost outlined in Equation~\ref{eq:optimization_problem}. From the vantage point of classifier $\mathcal{C}$, images transformed using $\mathcal{I}_k$ \emph{appear} to be closer to the \emph{home} distribution $D$. An example result of this is shown in Figure \ref{fig:example_sequence}.


\begin{figure*}
    \centering
    \includegraphics[width=0.7\textwidth]{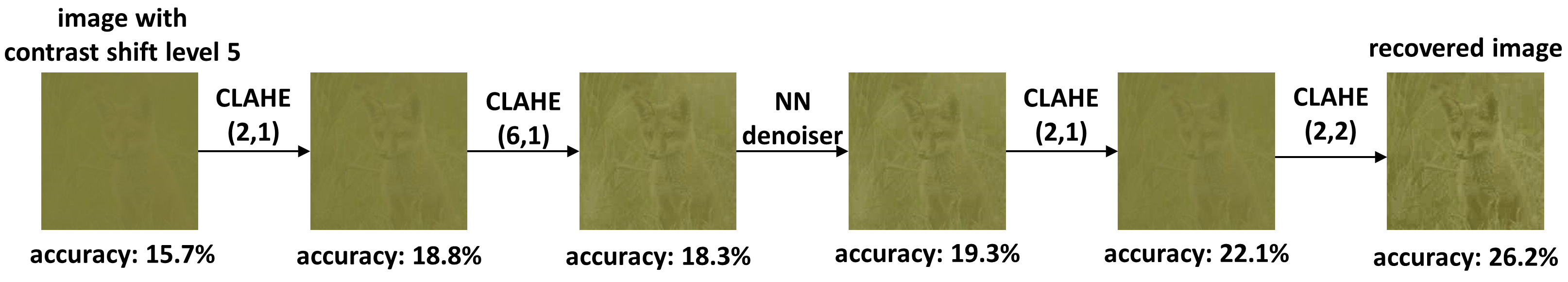}
    \vspace{-4mm}
    \caption{\footnotesize Example transformations applied to an image with contrast shift level 5. CLAHE($x$,$y$) denotes histogram equalization with strength determined by $x$ and $y$ (details in Appendix~\extref{sec:experiments}{E}\ext). The policy applies a non-trivial composition of transformations that would be difficult to find through manual manipulation. The policy chooses few redundant actions and improves the accuracy of an AugMix-trained ResNet-50 on a random batch of 1000 images. \vspace{-6mm} }
    \label{fig:example_sequence}
\end{figure*}

%% file: sections/training_using_RL.tex
\section{Distribution Shift Recovery is a Markov Decision Process}
\label{sec:training_rl}
Our task is to compute a composition of transforms $\mathcal{I}_k$ to apply to the corrupted set $\mathbf{V}_c$, sampled i.i.d from $D_\mathbb{T}$, to realize the optimization cost outlined in Equation \ref{eq:optimization_problem}. To this end, we note the following theorem.
\begin{theorem}\label{thm}
$R(\mathcal{I}_k) \leq \alpha \cdot  d_{TV}({D}, {D}_{\mathcal{I}_k \circ \mathbb{T}})$, for some finite $\alpha \in \real$ and semantic preserving transform $\mathcal{I}_k$.
\end{theorem}
\input{sections/proof_limit}

Thus, for a transform $\mathbb{T}$ (possibly a corruption), it is possible for the classifier to recover performance if the apparent distribution under $\mathcal{I}_k \circ \mathbb{T}$ is close enough to the original distribution $D$. Then the optimal $\mathcal{I}_k$ is the sequence of semantic preserving transforms that minimize the distance from $D$. Equivalently, this can be viewed as a sequence of actions maximizing a reward. More formally, the task of computing $\mathcal{I}_k$ can be formulated as a reactive policy for an MDP, which can be learned using standard reinforcement learning techniques. In this section, we present this MDP formulation in the context of an image classification task, but it may be applied in any learning setting.


\begin{definition}[\textbf{MDP}]
\label{def:mdp}
A \emph{Markov Decision Process (MDP)} is a $6$-tuple $\mathcal{E}=(\mathcal{S}, \mathbb{A}, \mathcal{P}, \mathcal{R}, \gamma, \mathrm{I}_0 )$, where  $\mathcal{S} \subseteq \real^n$ is the set of states, $\mathbb{A} \subseteq \real^m$ is the set of actions, $\mathcal{P}(s'|s,a)$ specifies the probability of transitioning from state $s$ to $s'$ on action $a$, $\mathcal{R}(s,a)$ is the reward returned when taking action $a$ from state $s$, $\gamma \in [0,1)$ is the discount factor, and $\mathrm{I}_0$ is the initial state distribution.
\end{definition}


\textbf{State Representation.}
The environment has access to a set $\mathbf{V}_c \subset \mathcal{X}$, which is a set of possibly corrupted images. 
A state of the MDP is a compressed representation of this set $\mathbf{V}_c$, capturing the type of corruptions in an image. 
Let us assume that this projection is captured by some function $\mathbb{F}_R : \mathcal{A}_{\mathcal{X}} \rightarrow \mathfrak{R}$,  where $\mathfrak{R}$ is the space of representations for a set of images. 
We want a representation $r = \mathbb{F}_R(\mathbf{V}_c)$ to be rich enough that a policy can decipher the appropriate choice of action in $\mathbb{A}$, but also compact enough that it is possible to learn a policy within a few episodes. Typically, a smaller state space size leads to faster convergence for reinforcement learning algorithms.


In this paper, for a set of images $\mathbf{V}_c$, we select a 3-dimensional state representation that measures the average brightness, standard deviation, and entropy of the images in $\mathbf{V}_c$. 
We convert each image $x$ to grayscale, then obtain the discrete wavelet transform and compute its average brightness $B(x)$, standard deviation $S(x)$, and entropy $E(x)$.
The state representation is an average of all these values across the images,
\begin{equation} \textstyle
    \mathbb{F}_R(\mathbf{V}_c) =\frac{1}{|\mathbf{V}_c|}\sum_{x\in\mathbf{V}_c} \begin{bmatrix}
     B(x), & S(x), & E(x)
    \end{bmatrix}^\top.
\end{equation}

\textbf{Actions and Transitions.} 
The set of actions $\mathbb{A} \subseteq \mathbb{S}$ is a set of semantic preserving transforms from which the learner chooses to maximize some reward. For an example, see the action set selected for the ImageNet-C benchmark in Appendix~\extref{sec:experiments}{E}\ext.
Hence, capturing Equation \ref{eq:optimization_problem} as a reward leads the agent to pick actions that mitigate the current corruption to some extent. Transitions model the effect of applying a transform from $\mathbb{S}$ to a possibly corrupted set $\mathbf{V}_c$. With slight abuse of notation, we use $\mathbb{T}(\mathbf{V}_c)$ to denote set $\{ v' : v' = \mathbb{T}(v), v \in \mathbf{V}_c \}$.

\begin{wrapfigure}{r}{0.48\textwidth}
    \vspace{-10mm}
    \begin{center}
    \includegraphics[width=7.4cm, height=4.7cm]{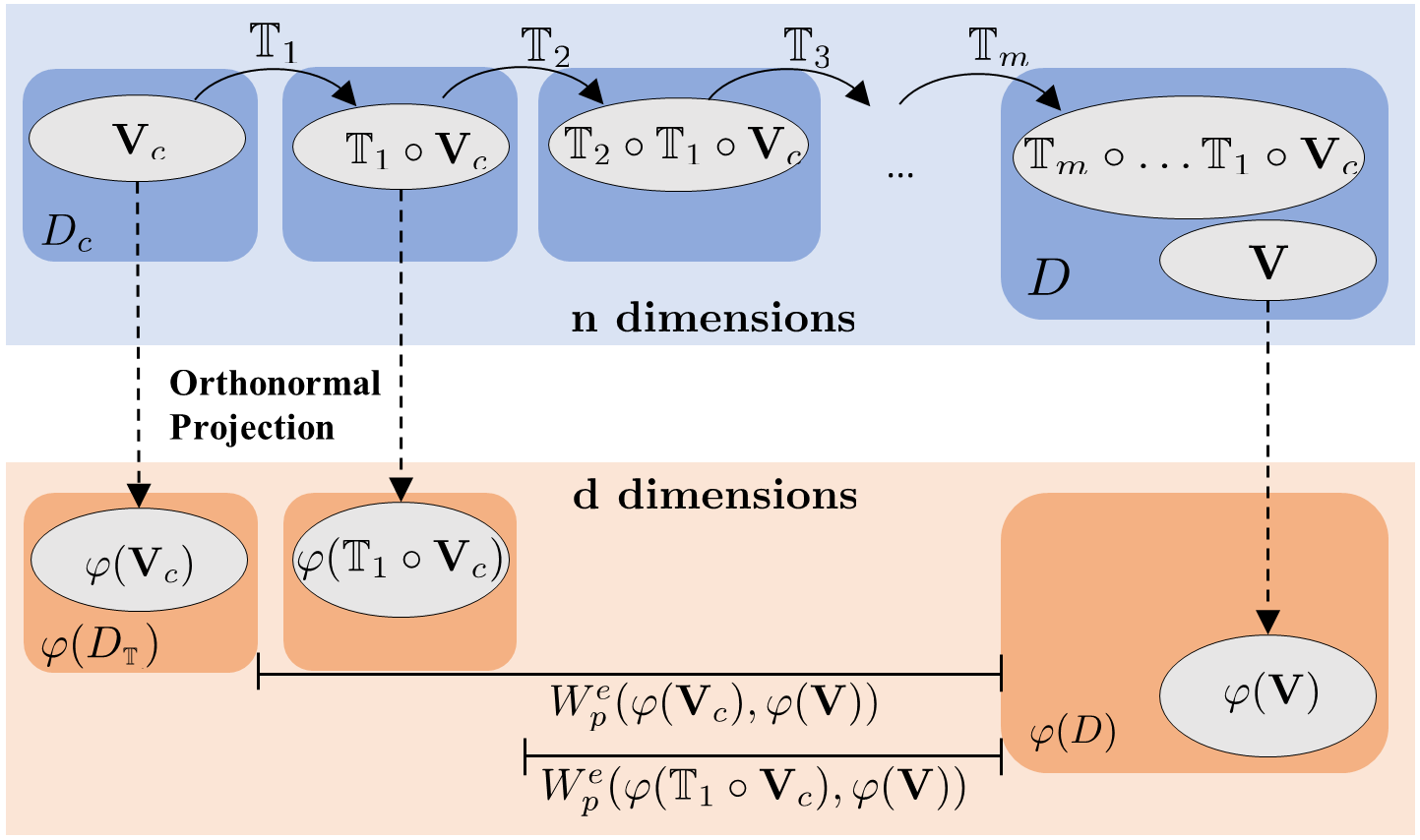}
    \end{center}
    \vspace{-8mm}
    \caption{\footnotesize Operation of \ours. Starting from $\mathbf{V}_c$, the algorithm selects a sequence of transforms $\mathcal{T}$ which move $\mathbf{V}_c$ closer to the original distribution $\mathbf{V}$. During the sequence, the orthonormal projections $\varphi({\mathbf{V}})$ and $\varphi({\mathbf{V}_c})$ are used to compute the Wasserstein distance $W_p(\varphi({\mathbf{V}_c}), \varphi(\mathbf{V}))$. See Section \ref{sec:orthonormal_proj} for details.
    \vspace{-4mm}
    }
    \label{fig:overall_diagram}
\end{wrapfigure}

\textbf{Computing Reward.} As shown in Figure \ref{fig:overall_diagram}, computing the reward for a set of images $\mathbf{V}_c$ corresponds to measuring the distance from a clean validation set $\mathbf{V}$. 
Ideally, the distance between the distributions from which $\mathbf{V}_c$ and $\mathbf{V}$ are drawn would be measured, but this is difficult without knowledge of the source distributions.
Instead, we use an empirical estimate of the Wasserstein distance~\citep{bonneel2011displacement} to compute the distributional distance between sets $\mathbf{V}_c$ and $\mathbf{V}$. It should be noted that a variety of distance functions can be employed here. We explore alternatives in Appendix \extref{ap:dist_functions}{A}\ext.

In practice the policy might not be able to reduce the $W_p^e$ to a level such that the classifier completely recovers the loss in accuracy. One reason for this is the possible non-existence of the inverse of the corruption transform. Another possible issue is that a transformation may overly alter the image such that the classifier performs poorly. Although we can combat against this by ensuring that actions make incremental changes to the image, this is hard to control. We therefore add a regularizer to the Wasserstein distance that penalizes excessive changes to the image. This is achieved using a visual similarity between pairs of images known as $ssim$ \citep{ssim}.

Given $\lambda > 0$ and $0\leq\omega<1$, the reward function is given by
\begin{equation} \textstyle
    R(s_t, a_t) = -W_p^e(\mathbf{V}, \mathbb{F}_R^{-1}(s_{t})) + \lambda L_S(\mathbb{F}_R^{-1}(s_{0}), \mathbb{F}_R^{-1}(s_{t+1})),
\end{equation}
where,\begin{equation}
\centering
\begin{aligned}
L_S(\mathbf{X}, \mathbf{Y}) = \begin{cases}
\log(1-ssim(\mathbf{X},\mathbf{Y})), & ssim(\mathbf{X},\mathbf{Y}) < \omega\\ 
0 & \text{otherwise}.
\end{cases}
\end{aligned}
\end{equation}
Note that the regularization hyperparameter $\lambda$ influences the level of aggression in correction. For example, selecting a large $\lambda$ favors actions with minimal influence on the data.


\textbf{Initial State.}
At test time, the initial state of the MDP is produced by a random environment corruption from the set $\mathbb{S}$ that the system is subjected to. In reality the designer does not have access to any of these corrupting transforms. Hence, at training time we train a policy network to reverse a set of \emph{surrogate} corruptions from $\mathbb{S}$, with the hope that some of these transfer at inference time to an unseen set of corruptions. An example surrogate corruption set and selection methodology for the ImageNet-C benchmark is given in Appendix~\extref{sec:experiments}{E}\ext.
We pick uniformly randomly from a finite set of $\mathsf{S}_c \subset \mathbb{S}$ of these surrogate corruptions to sample the initial state $\mathrm{I}_0$.

Standard RL techniques can be applied to learn a policy $\pi : \mathcal{S} \rightarrow \mathbb{A}$ for the MDP $\mathcal{E}$, which is a strategy to recover the distribution shift. The value of a state for policy $\pi$ is the expected return $\mathbf{R}_\pi(s_0)$, starting from state $s_0$, while executing the policy $\pi$ at every step, $a_t = \pi(s_t)$. 
The optimal policy $\pi^*$ maximizes the reward starting from the initial state distribution -- i.e., $\pi^*$ $=$ $\arg\max  V^{\mathrm{I}_0}(\pi)$, where $V^{\mathrm{I}_0}(\pi)$ $=$ $\mathbb{E}_{s_0 \in \mathrm{I}_0} [\mathbf{R}_\pi(s_0)]$.

%% file: sections/proof_limit.tex
\begin{proof}
Let us denote by $d$ the pdf of $D_{ \mathcal{I}_k \circ \mathbb{T}}$ and by $d^\prime$ the pdf of $D$. We begin by expanding the definition of $R(\mathcal{I}_k)$.
\begin{align*}
    R(\mathcal{I}_k) &= err(D_{ \mathcal{I}_k \circ \mathbb{T}}) - err(D) \\
    &= {\mathbb{E}}_{x_1 \sim D_{ \mathcal{I}_k \circ \mathbb{T}}, x_2 \sim D } \left[\mathbbm{1}( \mathcal{C}^* (x_1) \neq \mathcal{C}(x_1)) - \mathbbm{1}( \mathcal{C}^* (x_2) \neq \mathcal{C}(x_2))\right] \\
    &= \textstyle\int_\mathcal{X} \mathbbm{1}( \mathcal{C}^* (x) \neq \mathcal{C}(x))[d(x)-d^\prime(x)] \text{d}x \\
    &\leq \left|\textstyle\int_\mathcal{X} \left(\textstyle\sup_{x\in\mathcal{X}} \mathbbm{1}( \mathcal{C}^* (x) \neq \mathcal{C}(x)) \right) [d(x)-d^\prime(x) ] \text{d}x \right| \\
    &= \left| \left(\textstyle\sup_{x\in\mathcal{X}} \mathbbm{1}( \mathcal{C}^* (x) \neq \mathcal{C}(x)) \right) \textstyle\int_\mathcal{X}  [d(x)-d^\prime(x) ] \text{d}x \right| \\
    &\leq \left| \textstyle\sup_{x\in\mathcal{X}} \mathbbm{1}( \mathcal{C}^* (x) \neq \mathcal{C}(x))  \right| \cdot \textstyle\sup_{A\in\mathcal{A}_\mathcal{X}} \left| \textstyle\int_A  [d(x)-d^\prime(x) ] \text{d}x \right| \\
    &= \alpha \cdot  d_{TV} (D,D_{ \mathcal{I}_k \circ \mathbb{T}}),
\end{align*}
where $\alpha=\left| \sup_{x\in\mathcal{X}} \mathbbm{1}( \mathcal{C}^* (x) \neq \mathcal{C}(x))  \right|$. The first step assumes $\mathcal{X}$ is a continuous space. The rest follows from expressing the definition of computing expectation in terms of the classifier error.
\end{proof}

%% file: sections/correction_detection.tex
\section{Applying a Learned Policy for Detection and Correction}
\label{sec:detection_and_correction}

\input{sections/algorithm}

As described in Section \ref{sec:training_rl}, a DNN policy $\pi$ is trained at design time such that, given a corrupted set $\textbf{V}_c$ of i.i.d observations from corrupted distribution $D_\mathbb{T}$, the policy generates a finite composition $\mathcal{I}_k$ of $k<\infty$ transforms from $\mathbb{S}$. This solves the optimization problem in Equation \ref{eq:optimization_problem}. At design time, we assume the algorithm has access to a validation set $\mathbf{V} = \{v_1,\dots,v_n\}$, where $v_i$ is drawn i.i.d from the training distribution $D$. 

At inference / run-time \ours{} uses the corrupted set $\mathbf{V}_c = \{v^c_1, v^c_2, \dots, v^c_n \}$  drawn i.i.d from $D_{\mathbb{T}}$, $\mathbb{T} \in\mathbb{S}$, and the policy $\pi$ to select a correcting sequence $\mathcal{T} := \{ \mathbb{T}_k, \mathbb{T}_{k-1}, \dots \mathbb{T}_1 \}$.
We fix a maximum horizon $k$ to keep the algorithm tractable at both learning and deployment. The procedure for selecting sequence $\mathcal{T}$ is presented in Algorithm~\ref{alg:dsr}.

In Line \ref{line:init}, Algorithm \ref{alg:dsr} uses the estimator $\tilde{W}$ to estimate the initial Wasserstein  distance. Next, it runs policy $\pi$ for $k$ steps (Lines \ref{line:inoperable}$-$\ref{line:return}), where it iteratively applies the transform picked by the policy to update the corrupted set. The algorithm uses a lower dimensional state representation of $\mathbf{V}$ for evaluating the policy and for the estimator $\tilde{W}$ (see Section~\ref{sec:training_rl}). The algorithm collects and returns this set of transforms in $T$.

There are two stopping criteria to prevent the selection of damaging transforms. The first, in line \ref{line:inoperable}, checks for a minimum condition of operability on the set (see Section~\ref{sec:min_conditions}). The second, in Line \ref{line:stop}, guards against overly transforming images to diminishing returns. It also hedges against the chance that the distribution shift is not semantic preserving, which is always possible in reality. Hence, we pause the policy when the Wasserstein distance decreases beyond a threshold.

%% file: sections/algorithm.tex
\begin{wrapfigure}{r}{0.4\textwidth}
\vspace{-8mm}
\begin{minipage}[r]{0.4\textwidth}
\begin{algorithm}[H]
\caption{Transformation Selection}\label{alg:dsr}
\begin{algorithmic}[1]
\REQUIRE Validation set $\mathbf{V}$, corrupted set $\mathbf{V}_c$, policy $\pi$, \;horizon $k$, thresholds $\alpha,\beta\in(0,1]$
\STATE $w_0 \gets \tilde{W}(\mathbf{V}, \mathbf{V}_c)$ \label{line:init}
\STATE $\mathbf{V}_0 \gets \mathbf{V}_c$
\STATE $T \gets \{\;\}$
\FOR{$i=1,\dots,k$}

    \IF{$\mathbf{V}_i$ is inoperable} \label{line:inoperable}
        \STATE return $T$ 
    \ENDIF
    \STATE $\mathbb{T}_i \gets \pi(\mathbf{V}_{i-1})$
    \STATE $\mathbf{V}_i \gets \mathbb{T}_i(\mathbf{V}_{i-1})$
    \STATE $w_i \gets \tilde{W}(\mathbf{V}, \mathbf{V}_i)$
    \IF {$w_i \leq \alpha \,w_0 \vee w_i \geq \beta \,w_{i-1} $} \label{line:stop}
        \STATE return $T$
    \ENDIF
    \STATE $T \gets \{ \mathbb{T}_i \} \; \cup \; T$
\ENDFOR
\STATE Return $T$ \label{line:return}
\end{algorithmic}
\end{algorithm}
\vspace{-5mm}
\end{minipage}
\vspace{-5mm}
\end{wrapfigure}

%% file: sections/minimal_conditions.tex
\subsection{Minimum Condition for Operability}
\label{sec:min_conditions}

The optimal policy $\pi^*$ selects transformations that maximize the reward function, consisting of some distance function (e.g., the Wasserstein distance) plus a regularizing factor. The efficacy of such a policy thus depends on the distance function being a reliable estimate of a classifier's performance on the given data. We define an operable corrupted set $\textbf{V}_c$ as follows.

\begin{definition}[Operable Set]
\label{defn:operable_set}
For some distance function $d$, classifier $f$, validation set $\textbf{V}$, and set of transformations $\{\mathbb{T}_i\}_{i=1}^t$, a set $\textbf{V}_c = \{v^c_1,c^c_2,\dots,v^c_n\}$ is operable iff $\{d(\textbf{V}, \mathbb{T}_i(\textbf{V}_c))\}_{i=1}^t$ is negatively linearly correlated with $\{\frac{1}{n}\sum_{j=1}^n \mathbbm{1}[f(\mathbb{T}_i(v_j^c)) =  \mathcal{C}^*(\mathbb{T}_i(v_j^c))]\}_{i=1}^t$.
\end{definition}

Since the accuracy of the classifier we wish to adapt is unknown on inference-time data, operability is in practice difficult to evaluate. However, we can instead predict operability from the state representation $\mathbb{F}_R(\mathbf{V}_c)$ using a simple binary classifier, which is trained on surrogate corrupted data generated at design time. Label generation can be performed by evaluating the distance function and accuracies on these surrogate shifts, where $\mathbb{T}_i(\textbf{V}_c)$ are variations in the severity of shift approximating the effect of applying transformations. To evaluate accuracy, the selected classifier $f$ can be unique from the classifier we wish to adapt. This is motivated by the model collinearity phenomenon~\citep{mania2020classifier}, in which the relative accuracy on different distributions of data is often the same across multiple classifiers.

%% file: sections/orthogonal_projections.tex
\section{Dimensionality Reduction}
\label{sec:orthonormal_proj}

The empirical estimate of the Wasserstein distance converges in sample size to the true Wasserstein distance slowly in large dimensions~\citep{ramdas2017wasserstein}. 
Ideally, by reducing the dimensionality of the sample data, fewer samples are needed to achieve an accurate estimate of the Wasserstein distance.
We present rigorous theoretical justification that orthonormal projection is a reduction technique that preserves characteristics of the sample data at a distributional level.

Let $m,n\in\mathbb{N}$ and $p\in [1,\infty]$. Call $M(\mathbb{R}^n)$ and $M(\mathbb{R}^m)$ the spaces of probability measures on $\mathbb{R}^n$ and $\mathbb{R}^m$, respectively. Denote by $M^p(\mathbb{R}^n)$ and $M^p(\mathbb{R}^m)$ the spaces of probability measures having finite $p$-th moment on $\mathbb{R}^n$ and $\mathbb{R}^m$, respectively (here $p = \infty$ is interpreted in the limiting sense of essential supremum). For convenience, we consider only probability measures with densities, so that we do not have to check which measure is absolutely continuous to which other measure \cite[Section III]{lim2}. 

Suppose $m\leq n$ and consider the Stiefel manifold on $m\times n$ matrices with orthonormal rows.
$$O(m,n):=\{V\in\mathbb{R}^{m\times n} : VV^\top=I_d\}.$$
For any $V\in O(m,n)$ and $b\in \mathbb{R}^m$, let
$$\varphi_{V,b}:\mathbb{R}^n \rightarrow \mathbb{R}^m, \quad x \mapsto \varphi_{V,b}(x):=Vx+b,$$
and for any $\mu\in M(\mathbb{R}^n)$, let $\varphi_{V,b}(\mu):=\mu\circ\varphi_{V,b}^{-1}$ be the pushforward measure. This can be seen as a projection of $\mu$ onto the smaller dimensional space $\mathbb{R}^m$, and we call it a \textit{Cai-Lim projection}; it is not unique: it depends on the choice of $V$ and $b$. Recall then the definition of the $p$-Wasserstein distance between $\mu,\nu\in M^p(\mathbb{R}^n)$:
$W_p(\mu,\nu):=\left[ \inf_{\gamma\in\Gamma(\mu,\nu)} \int_{\mathbb{R}^{2n}} \| x-y \|_2^p \text{ d}\gamma(x,y) \right]^{\frac{1}{p}},$
 where $\|\cdot\|_2$ denotes the Euclidean norm and 
 $\Gamma(\mu,\nu):=\{\gamma\in M(\mathbb{R}^{2n}) : \pi_1^n(\gamma)=\nu \text{, } \pi_2^n(\gamma)=\mu\}$
is the set of couplings between $\mu$ and $\nu$, where $\pi_1^n$ is  the projection onto the first $n$ coordinates and $\pi_2^n$ is  the projection onto the last $n$ coordinates. The following is an important result.

\begin{lemma}\label{lemma_1}
    \textbf{\cite[Lemma II.1]{lim2}} Let $m,n\in\mathbb{N}$ and $p\in [1,\infty]$, and assume $m\leq n$. For any $\mu,\nu\in M^p(\mathbb{R}^n)$, any $V\in O(m,n)$, and any $b\in\mathbb{R}^m$, we have that 
    $$W_p \left( \varphi_{V,b}(\mu),\varphi_{V,b}(\nu) \right) \leq W_p \left( \mu,\nu \right).$$
\end{lemma}
This can be interpreted as ``losing some information'' when performing a Cai-Lim projection: in smaller dimensional spaces, distributions $\mu$ and $\nu$ seem to be closer than they actually are. This is an inevitable byproduct of any projection operation. Lemma~\ref{lemma_1} implies the following corollary.

\begin{corollary}\label{cor_lemma1}
    Let $m,n\in\mathbb{N}$ and $p\in [1,\infty]$, and assume $m\leq n$. Consider $\mu,\nu,\rho,\zeta\in M^p(\mathbb{R}^n)$, and pick any $V\in O(m,n)$ and any $b\in\mathbb{R}^m$. Suppose $W_p(\mu,\nu)\geq W_p(\rho,\zeta)$. Then, there exists $\varepsilon>0$ such that if $W_p \left( \mu,\nu \right) - W_p \left( \varphi_{V,b}(\mu),\varphi_{V,b}(\nu) \right) \leq \varepsilon$, then 
    $$W_p\left( \varphi_{V,b}(\mu),\varphi_{V,b}(\nu) \right)\geq W_p\left( \varphi_{V,b}(\rho),\varphi_{V,b}(\zeta) \right).$$
\end{corollary}
\begin{proof} Set $\varepsilon=W_p \left( \rho,\zeta \right) - W_p \left( \varphi_{V,b}(\rho),\varphi_{V,b}(\zeta) \right)$. The result follows immediately by Lemma \ref{lemma_1}.
\end{proof}

Corollary \ref{cor_lemma1} states the following. If we ``do not lose too much information'' when performing a Cai-Lim projection of the two farthest apart distributions,
then the inequality $W_p(\mu,\nu)\geq W_p(\rho,\zeta)$ between the original distribution is preserved between their projections. A more intuitive discussion, as well as empirical justification, can be found in  Appendices~\extref{ap:orth_proj_intuition}{B} and~\extref{ap:orth_proj}{C}\ext.





%% file: sections/background.tex
\section{Related Work}
Distribution shifts can make deep learning models act dangerously in the real world \citep{conformance-paper}. Some offline techniques aim to improve classification robustness to distribution shift at training time. These include data augmentation, which expands the experiences that a model would be exposed to during training~\citep{hendrycks2019augmix,noisy_mix_paper,manifold_mix,cut_mix,puzzle_mix,deep_augment}. Another offline approach is domain adaptation, in which models are adapted to perform well on unlabeled data in some new distribution \citep{ben2010theory, bousmalis2017unsupervised, hoffman2018cycada}. Offline approaches lack flexibility to shifts unforeseen at train time, but they can be combined with online methods like \ours{} to further improve performance.

Alternate techniques aim to handle distribution shift online. For example, test-time adaptation methods perform additional training at inference time in response to incoming data \citep{gandelsman2022test, wang2022continual, mummadi2021test}. Test time augmentation approaches instead apply transformations to the input, then ensemble predictions if multiple transformations are applied \citep{simonyan2014very, krizhevsky2012imagenet, he2016deep, guo2017countering, mummadi2021test, kim2020learning, lyzhov2020greedy}. In contrast, \ours{} selects transformations online without querying the model. Furthermore, our work uniquely identifies MDPs as an equivalent formulation of the transformation selection problem, enabling the application of standard RL techniques. We further explore related work in Appendix~\extref{sec:related}{D}\ext.

%% file: sections/application.tex
\section{Application: ImageNet-C} 
\label{sec:application} \label{sec:policy_net}
The ImageNet-C dataset~\citep{hendrycks2019benchmarking} is constructed from ImageNet samples corrupted by 19 semantic preserving transformations. We deploy the learned policy $\pi$ in \ours{} to correct ImageNet-C corruptions, and we evaluate the accuracies of Resnet-50 classifiers with and without correction. We evaluate a baseline classifier trained without data augmentation and classifiers trained with data augmentation through AugMix, NoisyMix, DeepAugment, DeepAugment with Augmix, and Puzzlemix.
We also evaluate the ability to generalize outside of the ImageNet-C benchmark. Experiment details can be found in Appendix~\extref{sec:experiments}{E}\ext.\footnote{Our code can be found at \href{https://github.com/vwlin/SuperStAR}{https://github.com/vwlin/SuperStAR}.}

\input{tables/main_imagenet_c}
\input{sections/results}

\input{tables/imagenet_cp}
\input{tables/main_cifar_c_0-972}

%% file: tables/main_imagenet_c.tex
\begin{table*}[t]\scriptsize 
\caption{\footnotesize Average accuracies (\%) on each ImageNet-C shift with and without \ours{} for ResNet-50 classifiers. Accuracy improvement is denoted by $\Delta = \text{R (recovered)} - \text{S (shifted)}$. Values are over 5 severity levels with 3 trials each. "gaussian" refers to Gaussian noise.}
\vspace{-5mm}
\label{tab:imagenet_c_results}
\setlength{\tabcolsep}{1.7pt}
\begin{center}
\scriptsize
\begin{tabular}{|p{9ex}|p{5.2ex} p{5.2ex} p{5ex}|p{5.2ex} p{5.2ex} p{5ex}|p{5.2ex} p{5.2ex} p{5ex}|p{5.2ex} p{5.2ex} p{5ex}|p{5.2ex} p{5.2ex} p{5ex}|p{5.2ex} p{5.2ex} p{5ex}|}
    \cline{2-19}
    \multicolumn{1}{c}{} & \multicolumn{3}{|c|}{No Data Aug} & \multicolumn{3}{|c|}{AugMix} & \multicolumn{3}{|c|}{NoisyMix} & \multicolumn{3}{|c|}{DeepAugment} & \multicolumn{3}{|c|}{DeepAug+AugMix} & \multicolumn{3}{|c|}{PuzzleMix}\\
    \hline
    shift &
    S & R & $\Delta$ &
    S & R & $\Delta$ &
    S & R & $\Delta$ &
    S & R & $\Delta$ &
    S & R & $\Delta$ &
    S & R & $\Delta$\\
    \hline
    none &
    74.52 & 74.52 & 0.00 &
    75.94 & 75.94 & 0.00 &
    76.22 & 76.22 & 0.00 &
    75.86 & 75.86 & 0.00 &
    75.26 & 75.26 & 0.00 &
    75.63 & 75.63 & 0.00\\
    \hline
    gaussian &
    31.11 & 43.44 & \textcolor{red}{12.33} &
    41.90 & 50.87 & \textcolor{red}{8.98} &
    52.71 & 55.51 & \textcolor{red}{2.80} &
    59.07 & 59.48 & \textcolor{red}{0.41} &
    55.39 & 61.43 & \textcolor{red}{6.05} &
    41.48 & 46.94 & \textcolor{red}{5.46}\\
    \hline
    shot &
    28.61 & 42.81 & \textcolor{red}{14.21} &
    41.78 & 50.93 & \textcolor{red}{9.15} &
    51.81 & 55.38 & \textcolor{red}{3.57} &
    58.21 & 58.46 & \textcolor{red}{1.24} &
    55.76 & 62.37 & \textcolor{red}{6.61} &
    37.39 & 45.56 & \textcolor{red}{7.77}\\
    \hline
    impulse &
    26.57 & 39.36 & \textcolor{red}{12.79} &
    38.78 & 47.49 & \textcolor{red}{8.71} &
    50.73 & 53.37 & \textcolor{red}{2.64} &
    58.61 & 58.38 & -0.23 &
    55.16 & 60.67 & \textcolor{red}{5.50} &
    35.28 & 42.82 & \textcolor{red}{7.54}\\
    \hline
    snow &
    30.51 & 29.28 & -1.13 &
    37.89 & 36.85 & -1.04 &
    43.20 & 41.52 & -1.68 &
    41.71 & 39.91 & -1.80 &
    47.68 & 46.36 & -1.32 &
    39.48 & 37.74 & -1.75\\
    \hline
    frost &
    35.16 & 35.07 & -0.09 &
    41.39 & 41.24 & -0.15 &
    50.05 & 49.46 & -0.59 &
    46.87 & 46.08 & -0.78 &
    51.21 & 50.26 & -0.95 &
    46.96 & 45.75 & -1.21\\
    \hline
    brightness &
    65.17 & 65.72 & \textcolor{red}{0.55} &
    67.35 & 68.46 & \textcolor{red}{1.11} &
    68.82 & 69.87 & \textcolor{red}{1.05} &
    69.04 & 69.73 & \textcolor{red}{0.69} &
    69.42 & 70.18 & \textcolor{red}{0.76} &
    69.59 & 69.67 & \textcolor{red}{0.08}\\
    \hline
    contrast &
    35.56 & 37.69 & \textcolor{red}{2.14} &
    48.96 & 49.85 & \textcolor{red}{0.89} &
    50.37 & 52.74 & \textcolor{red}{2.37} &
    44.89 & 48.23 & \textcolor{red}{3.33} &
    56.01 & 57.40 & \textcolor{red}{1.39} &
    50.56 & 52.87 & \textcolor{red}{2.30}\\
    \hline
    speckle &
    36.09 & 49.26 & \textcolor{red}{13.18} &
    50.61 & 56.94 & \textcolor{red}{6.33} &
    57.67 & 60.88 & \textcolor{red}{3.22} &
    62.21 & 63.81 & \textcolor{red}{1.59} &
    60.93 & 65.66 & \textcolor{red}{4.74} &
    42.24 & 51.92 & \textcolor{red}{9.68}\\
    \hline
    spatter &
    46.65 & 46.44 & -0.20 &
    53.25 & 52.93 & -0.32 &
    57.63 & 57.34 & -0.29 &
    53.74 & 53.58 & -0.16 &
    57.75 & 57.61 & -0.14 &
    53.27 & 52.95 & -0.32\\
    \hline
    saturate &
    59.00 & 59.17 & \textcolor{red}{0.17} &
    61.42 & 61.89 & \textcolor{red}{0.47} &
    63.48 & 64.00 & \textcolor{red}{0.52} &
    64.59 & 64.81 & \textcolor{red}{0.22} &
    65.79 & 66.12 & \textcolor{red}{0.33} &
    65.96 & 65.60 & -0.37\\
    \hline
\end{tabular}
\end{center}
\vspace{-5mm}
\end{table*}

%% file: sections/results.tex
\textbf{ImageNet-C.} Table~\ref{tab:imagenet_c_results} summarizes the classifier accuracy improvements from applying \ours{} to ImageNet-C. For brevity, we exclude from Table~\ref{tab:imagenet_c_results} shifts for which \ours{} incurs no change in accuracy (see Appendix~\extref{sec:supp_imagenet_c_table}{F}\ext{} for full table). Corrections via our \ours{} algorithm lead to accuracy improvements for a majority of shifts, with maximum improvement of \maxICacc{} (averaged across all five severity levels). In general accuracy improvements are greater for higher severities (due to space constraints, per-severity accuracy improvements are shown in Appendix~\extref{sec:supp_imagenet_c_table}{F}\ext). Furthermore, when combined with data augmentation, \ours{} often leads to higher accuracies than \ours{} or data augmentation alone. In the cases of no shift and the shifts not shown in Table~\ref{tab:imagenet_c_results}, \ours{} refrains from taking any action and does not affect accuracy. This is examined further in Appendix~\extref{sec:inaction}{G}\ext.

Interestingly, for a hyperparameter selection that allows some increase in Wasserstein distance at each step ($\beta = 1.12$), we find that \ours{} selects a non-trivial 5-action sequence of transformations for contrast shift severity level 5, which incrementally increases the accuracy of the AugMix classifier. Figure~\ref{fig:example_sequence} shows a sample image with the applied transformations and resulting accuracies. The transformations incur a noticeable change in the image.

\textbf{Generalization beyond ImageNet-C.} We also evaluate the ability of \ours{} to generalize outside of the ImageNet-C benchmark. 1) We construct composite ImageNet-C shifts from pairs of the ImageNet-C corruptions for which \ours{} improves accuracy. To each shift, we reapply our operability classifier and policy network without retraining. Table~\ref{tab:imagenet_cp_results} shows the average classifier accuracy improvements from \ours{}. For nearly all shifts, our method improves classifier accuracy, with a maximum improvement in average accuracy of \maxICPacc{} (see Appendix~\extref{sec:composite_imagenetc}{H}\ext{} for severity level breakdowns). 2) We also apply our method to CIFAR-100-C~\citep{hendrycks2019benchmarking}, an analagous benchmark to ImageNet-C. We use the pretrained policy network and retrain only the operability classifier on the surrogate corruptions regenerated for CIFAR-100. We evaluate on a variety of Wide ResNets trained with data augmentation (AugMix, NoisyMix, and PuzzleMix) and without (baseline).\footnote{DeepAugment and DeepAugment with Augmix are not evaluated on CIFAR-100 in the original publications.} Appendix~\extref{sec:cifar_details}{I}\ext{} contains further details. Table~\ref{tab:cifar_results} shows the accuracy improvements from applying \ours{} to CIFAR-100-C. Without any retraining of the policy network, \ours{} improves average accuracy by up to \maxCCacc{} (see Appendix~\extref{sec:cifar100c}{J}\ext{} for full table and severity level breakdowns). In some cases, such as impulse noise, our method decreases accuracy for the classifiers trained with data augmentation, while increasing that of the baseline classifier. We attribute this to the operability classifier mislabeling such shifts as operable.

Overall, \ours{} demonstrates an ability to dynamically respond to distribution shift online, selecting context-appropriate actions (or inaction) and significantly improving classification accuracy on ImageNet-C for a variety of corruptions. In some cases, \ours{} identifies complex sequences of corrective transformations. We attribute this strong performance largely to our selection of surrogate shifts (see Appendix~\extref{sec:clustering}{K}\ext{} for more discussion). \ours{} also generalizes to distribution shifts outside of the ImageNet-C benchmark, without retraining the policy network. Finally, we note that although promising, \ours{} is limited in its speed of response and its reliance on appropriately selected actions and surrogate corruptions. We further discuss these limitations and more in Appendix~\extref{sec:limitations}{L}\ext.

%% file: tables/imagenet_cp.tex
\begin{table*}[t]\scriptsize 
\caption{\footnotesize Average accuracies (\%) on each composite ImageNet-C shift with and without \ours{} for ResNet-50 classifiers. Composites are combinations of Gaussian noise (GN), shot noise (ShN), impulse noise (IN), speckle noise (SpN) with brightness (B), contrast (C), saturate (S). Accuracy improvement is denoted $\Delta = \text{R (recovered)} - \text{S (shifted)}$. Values are over 5 severity levels with 3 trials each.}
\vspace{-5mm}
\label{tab:imagenet_cp_results}
\setlength{\tabcolsep}{1.7pt}
\begin{center}
\scriptsize
\begin{tabular}{|p{9ex}|p{5.2ex} p{5.2ex} p{5ex}|p{5.2ex} p{5.2ex} p{5ex}|p{5.2ex} p{5.2ex} p{5ex}|p{5.2ex} p{5.2ex} p{5ex}|p{5.2ex} p{5.2ex} p{5ex}|p{5.2ex} p{5.2ex} p{5ex}|}
    \cline{2-19}
    \multicolumn{1}{c}{} & \multicolumn{3}{|c|}{No Data Aug} & \multicolumn{3}{|c|}{AugMix} & \multicolumn{3}{|c|}{NoisyMix} & \multicolumn{3}{|c|}{DeepAugment} & \multicolumn{3}{|c|}{DeepAug+AugMix} & \multicolumn{3}{|c|}{PuzzleMix}\\
    \hline
    shift &
    S & R & $\Delta$ &
    S & R & $\Delta$ &
    S & R & $\Delta$ &
    S & R & $\Delta$ &
    S & R & $\Delta$ &
    S & R & $\Delta$\\
    \hline
    GN + B &
    23.75 & 30.65 & \textcolor{red}{6.90} &
    29.03 & 37.20 & \textcolor{red}{8.17} &
    42.87 & 44.37 & \textcolor{red}{1.50} &
    51.29 & 49.80 & -1.49 &
    43.03 & 52.17 & \textcolor{red}{9.11} &
    33.25 & 35.15 & \textcolor{red}{1.90}\\
    \hline
    GN + C &
    22.39 & 22.39 & 0.00 &
    29.40 & 29.40 & 0.00 &
    35.96 & 35.96 & 0.00 &
    34.00 & 34.00 & 0.00 &
    41.31 & 41.31 & 0.00 &
    32.46 & 32.46 & 0.00\\
    \hline
    GN + S &
    19.35 & 26.37 & \textcolor{red}{7.02} &
    26.65 & 33.43 & \textcolor{red}{6.79} &
    37.13 & 38.50 & \textcolor{red}{1.37} &
    42.83 & 44.50 & \textcolor{red}{1.68} &
    42.41 & 50.82 & \textcolor{red}{8.41} &
    28.96 & 30.98 & \textcolor{red}{2.02}\\
    \hline
    IN + B &
    22.03 & 30.08 & \textcolor{red}{8.04} &
    29.62 & 37.91 & \textcolor{red}{8.29} &
    43.66 & 45.89 & \textcolor{red}{2.24} &
    52.47 & 52.90 & \textcolor{red}{0.43} &
    44.14 & 53.95 & \textcolor{red}{9.81} &
    32.13 & 35.65 & \textcolor{red}{3.52} \\
    \hline
    IN + C &
    18.41 & 18.41 & 0.00 &
    25.89 & 25.89 & 0.00 &
    33.05 & 33.05 & 0.00 &
    31.61 & 31.61 & 0.00 &
    39.72 & 39.72 & 0.00 &
    26.27 & 26.27 & 0.00\\
    \hline
    IN + S &
    17.59 & 24.61 & \textcolor{red}{7.02} &
    25.24 & 30.37 & \textcolor{red}{5.12} &
    34.23 & 36.07 & \textcolor{red}{1.84} &
    40.96 & 42.34 & \textcolor{red}{1.37} &
    43.58 & 49.05 & \textcolor{red}{5.47} &
    25.74 & 29.37 & \textcolor{red}{3.62}\\
    \hline
    ShN + B &
    22.32 & 30.08 & \textcolor{red}{7.76} &
    28.38 & 35.82 & \textcolor{red}{7.44} &
    40.26 & 41.64 & \textcolor{red}{1.38} &
    49.44 & 47.65 & -1.79 &
    41.38 & 50.61 & \textcolor{red}{9.23} &
    31.84 & 33.63 & \textcolor{red}{1.79}\\
    \hline
    ShN + C &
    21.13 & 21.03 & -0.10 &
    28.36 & 27.94 & -0.42 &
    35.22 & 34.98 & -0.24 &
    34.26 & 33.96 & -0.30 &
    41.05 & 40.76 & -0.29 &
    30.13 & 30.26 & \textbf{0.13}\\
    \hline
    ShN + S &
    18.49 & 26.03 & \textcolor{red}{7.54} &
    27.33 & 33.74 & \textcolor{red}{6.40} &
    37.58 & 39.11 & \textcolor{red}{1.52} &
    42.47 & 44.84 & \textcolor{red}{2.37} &
    44.04 & 51.26 & \textcolor{red}{7.22} &
    26.33 & 30.03 & \textcolor{red}{3.70}\\
    \hline
    SpN + B &
    27.20 & 32.43 & \textcolor{red}{5.23} &
    32.64 & 38.55 & \textcolor{red}{5.91} &
    43.60 & 45.64 & \textcolor{red}{2.05} &
    52.77 & 52.37 & -0.41 &
    45.82 & 52.99 & \textcolor{red}{7.17} &
    35.80 & 37.81 & \textcolor{red}{2.01}\\
    \hline
    SpN + C &
    24.00 & 24.18 & \textcolor{red}{0.18} &
    32.70 & 32.58 & -0.12 &
    39.21 & 39.57 & \textcolor{red}{0.36} &
    36.30 & 36.98 & \textcolor{red}{0.68} &
    44.19 & 44.49 & \textcolor{red}{0.30} &
    33.40 & 33.74 & \textcolor{red}{0.34}\\
    \hline
    SpN + C &
    23.80 & 31.36 & \textcolor{red}{7.56} &
    35.87 & 40.22 & \textcolor{red}{4.35} &
    45.04 & 45.98 & \textcolor{red}{0.93} &
    47.64 & 50.87 & \textcolor{red}{3.23} &
    50.25 & 55.98 & \textcolor{red}{5.72} &
    30.55 & 35.94 & \textcolor{red}{5.39}\\
    \hline
\end{tabular}
\end{center}
\vspace{-5mm}
\end{table*}

%% file: tables/main_cifar_c_0-972.tex
\begin{table*}[t]\scriptsize 
\caption{\footnotesize Average accuracies (\%) on each CIFAR-100-C shift with and without \ours{} for Wide ResNet classifiers. Accuracy improvement is denoted by $\Delta = \text{R (recovered)} - \text{S (shifted)}$.
Values are over 5 severity levels with 3 trials each.}
\vspace{-5mm}
\label{tab:cifar_results}
\setlength{\tabcolsep}{2pt}
\begin{center}
\scriptsize
\begin{tabular}{|p{13ex}|p{5.5ex} p{5.5ex} p{5ex}|p{5.5ex} p{5.5ex} p{5ex}|p{5.5ex} p{5.5ex} p{5ex}|p{5.5ex} p{5.5ex} p{5ex}|p{5.5ex} p{5.5ex} p{5ex}|p{5.5ex} p{5.5ex} p{5ex}|}
    \cline{2-13}
    \multicolumn{1}{c}{} & \multicolumn{3}{|c|}{No Data Aug} & \multicolumn{3}{|c|}{AugMix} & \multicolumn{3}{|c|}{NoisyMix}  & \multicolumn{3}{|c|}{PuzzleMix}\\
    \hline
    shift &
    S & R & $\Delta$ &
    S & R & $\Delta$ &
    S & R & $\Delta$ &
    S & R & $\Delta$ \\
    \hline
    none &
    81.13 & 81.13 & 0.00 &
    76.28 & 76.28 & 0.00 &
    81.29 & 81.29 & 0.00 &
    84.01 & 84.01 & 0.00 \\
    \hline
    gaussian noise &
    21.12 & 26.81 & \textcolor{red}{5.70} &
    47.89 & 51.07 & \textcolor{red}{3.18} &
    65.91 & 66.34 & \textcolor{red}{0.43} &
    20.87 & 28.18 & \textcolor{red}{7.31} \\
    \hline
    shot noise &
    29.96 & 36.34 & \textcolor{red}{6.38} &
    55.69 & 58.24 & \textcolor{red}{2.55} &
    70.39 & 70.72 & \textcolor{red}{0.33} &
    31.12 & 39.37 & \textcolor{red}{8.25} \\
    \hline
    impulse noise &
    19.21 & 26.09 & \textcolor{red}{6.88} &
    59.68 & 59.09 & -0.59 &
    79.72 & 76.08 & -3.64 &
    37.18 & 37.01 & -0.17 \\
    \hline
    glass blur &
    20.68 & 26.61 & \textcolor{red}{5.93} &
    54.08 & 56.09 & \textcolor{red}{2.00} &
    58.82 & 60.48 & \textcolor{red}{1.66} &
    31.07 & 37.62 & \textcolor{red}{6.55} \\
    \hline
    speckle noise &
    31.58 & 35.76 & \textcolor{red}{4.18} &
    58.11 & 59.33 & \textcolor{red}{1.23} &
    71.67 & 71.57 & -0.09 &
    33.96 & 38.94 & \textcolor{red}{4.98} \\
    \hline
    spatter &
    61.23 & 62.23 & \textcolor{red}{1.00} &
    72.28 & 71.25 & -1.02 &
    78.11 & 77.44 & -0.67 &
    79.73 & 78.90 &  -0.83 \\
    \hline
    saturate &
    68.82 & 68.88 & \textcolor{red}{0.06} &
    64.52 & 64.77 & \textcolor{red}{0.25} &
    69.83 & 70.21 & \textcolor{red}{0.39} &
    72.93 & 72.99 & \textcolor{red}{0.05} \\
    \hline
\end{tabular}
\end{center}
\vspace{-5mm}
\end{table*}

%% file: sections/conclusion.tex
\section{Conclusion}
In this work we presented \ours{}, which uses control for learned models to detect and recover from distribution shift. \ours{} uses the Wasserstein distance (with a theoretically-sound approach for dimensionality reduction using orthonormal projections) to detect distribution shifts and select recovery actions from a library of image correction techniques. We formulate this action selection problem as a Markov decision process, and we train the policy for computing actions using reinforcement learning. To hedge against harmful actions, we employ a binary classifier to check a minimum condition for our method to operate on corrupted data. We applied our approach to various classifiers on the ImageNet-C dataset, and we obtained significant accuracy improvements when compared to the classifiers alone. Finally, we showed that \ours{} generalizes to composite ImageNet-C shifts and CIFAR-100-C with no retraining of the policy. Expansion of the action library and additional tuning can lead to further improvements on these benchmarks.

%% file: supplement.tex
\clearpage
\appendix

\section{Alternative Distance Functions}
\label{ap:dist_functions}
\input{supp_sections/distance_functions}

\vfill

\newpage
\section{An Intuitive View of Orthonormal Projection for the Wasserstein Distance}
\label{ap:orth_proj_intuition}
\input{supp_sections/orth_projection_intuition}
\vfill

\newpage
\section{Empirical Evaluation of Orthonormal Projection}\label{ap:orth_proj}
Figures~\ref{fig:gauss-noise_proj},~\ref{fig:imp-noise_proj}, and~\ref{fig:gauss-blur_proj} compare the empirical Wasserstein distance using orthonomal projection to Gaussian random projection and sparse random projection when MNIST is perturbed with additive Gaussian noise, additive impulse noise, and Gaussian blur, respectively. We find that orthonormal projection better preserves distributional information than the alternatives.
\begin{figure}[ht] 
    \centering
    \includegraphics[width=0.8\textwidth]{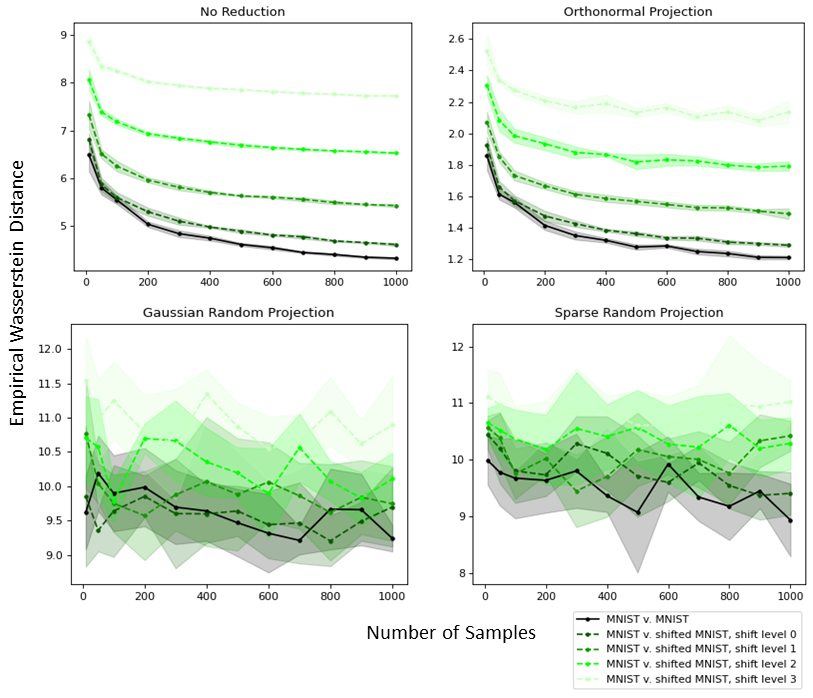}
    \caption{\small Empirical Wasserstein distance between MNIST and MNIST with varied levels of additive Gaussian noise, measured over a range of sample sizes. Curves are taken over 5 trials. MNIST samples are downsampled to 24$\times$24, flattened, and projected to 50 dimensions.
    Orthonormal projection better preserves distributional information than Gaussian random projection and sparse random projection.}\label{fig:gauss-noise_proj}
\end{figure}
\vfill

\newpage
\begin{figure}[ht]
    \centering
    \includegraphics[width=0.8\textwidth]{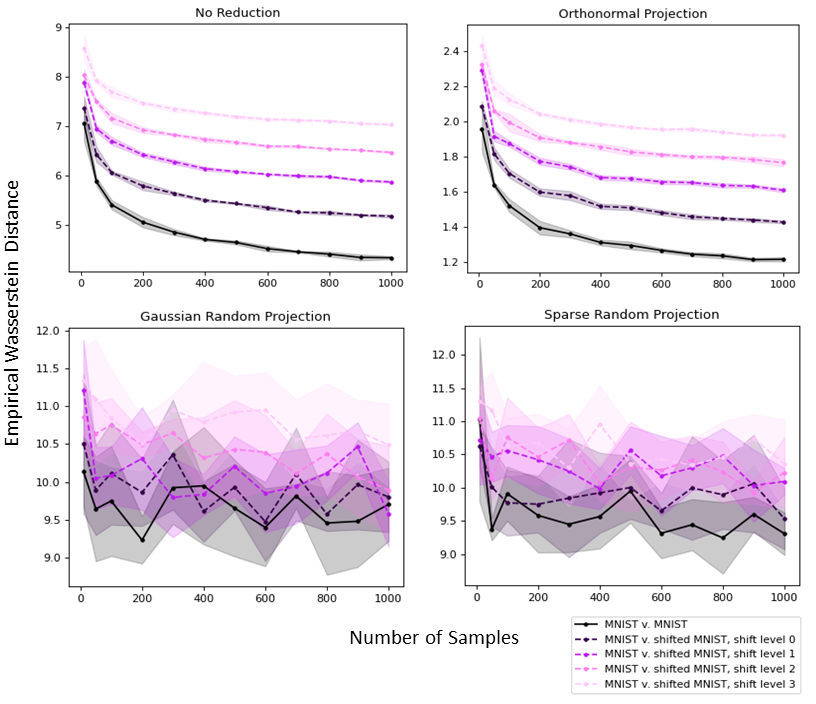}
    \caption{\small Empirical Wasserstein distance between MNIST and MNIST with varied levels of additive impulse noise, measured over a range of sample sizes. Curves are taken over 5 trials. MNIST samples are downsampled to 24$\times$24, flattened, and projected to 50 dimensions.
    Orthonormal projection better preserves distributional information than Gaussian random projection and sparse random projection.}\label{fig:imp-noise_proj}
\end{figure}
\vfill

\newpage
\begin{figure}[ht]
    \centering
    \includegraphics[width=0.8\textwidth]{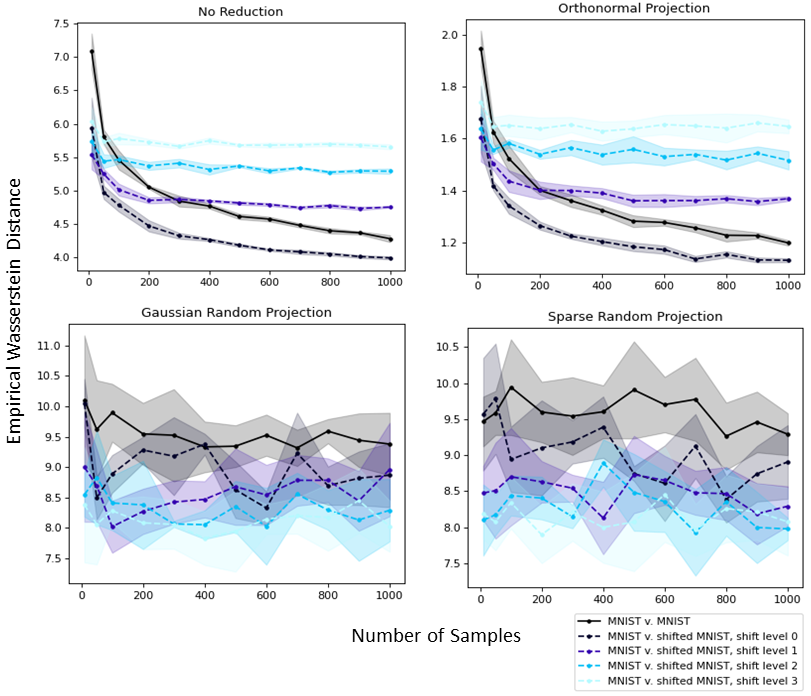}
    \caption{\small Empirical Wasserstein distance between MNIST and MNIST with varied levels of Gaussian blur, measured over a range of sample sizes. Curves are taken over 5 trials. MNIST samples are downsampled to 24$\times$24, flattened, and projected to 50 dimensions.
    Orthonormal projection better preserves distributional information than Gaussian random projection and sparse random projection.}\label{fig:gauss-blur_proj}
\end{figure}
\vfill

\newpage
\section{Related Work}
\label{sec:related}
\input{supp_sections/related_work}
\vfill

\newpage
\section{Experimental Details}
\label{sec:experiments}
\input{supp_sections/experiments}
\vfill





\newpage
\section{Full Experimental Results for \ours{} on ImageNet-C Shifts}
\label{sec:supp_imagenet_c_table}
\input{tables/supp_imagenet_c}
\begin{figure}[ht]
\centering\includegraphics[width=\linewidth]{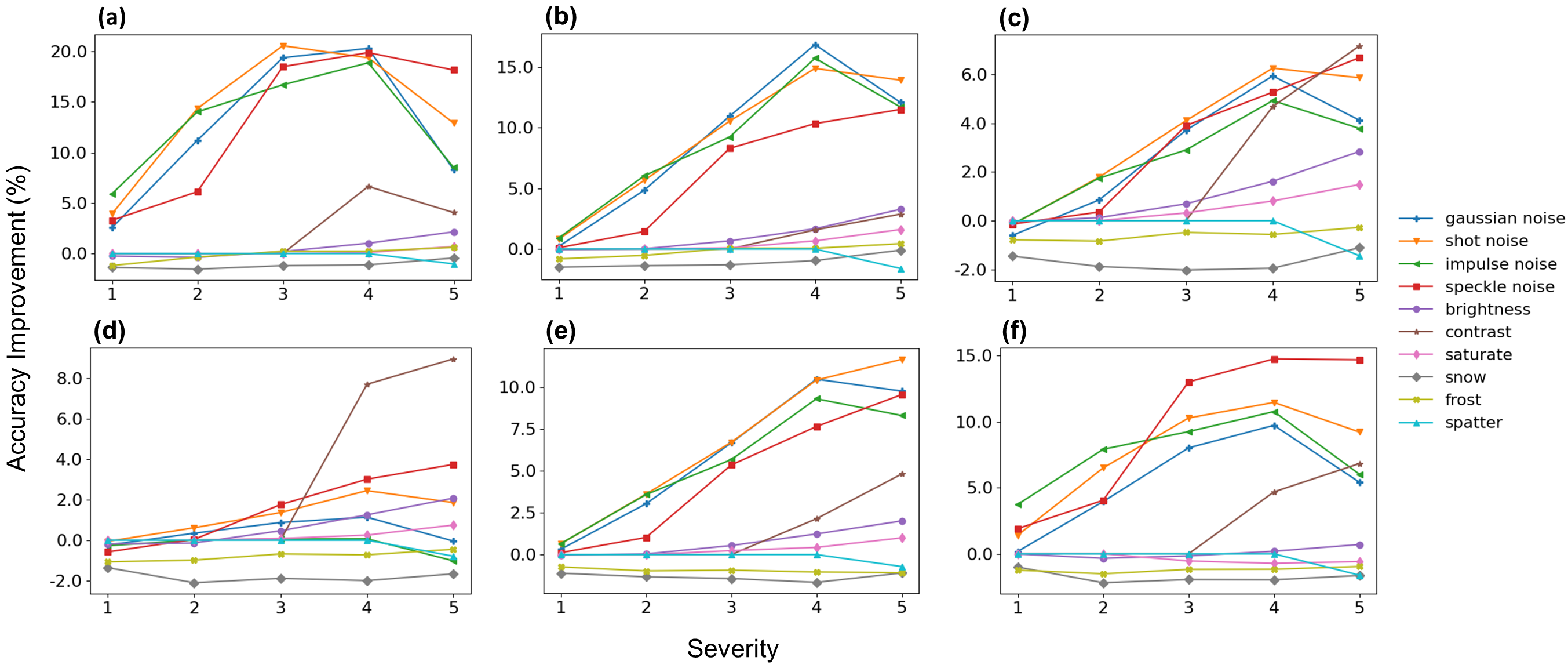}
    \caption{\small Accuracy improvements for increasing severity evaluated on classifiers trained with a) no data augmentation, b) AugMix, c) NoisyMix, d) DeepAugment, e) DeepAugment and AugMix, and f) PuzzleMix. In general, performance improvements are more pronounced for greater severities of corruption.}
\end{figure}
\vfill


\newpage
\section{Inaction for Select Shifts}
\label{sec:inaction}
\input{supp_sections/inaction}
\vfill

\newpage
\section{Full Experimental Results for \ours{} on Composite ImageNet-C Shifts}
\label{sec:composite_imagenetc}
\begin{figure}[ht]
\centering\includegraphics[width=\linewidth]{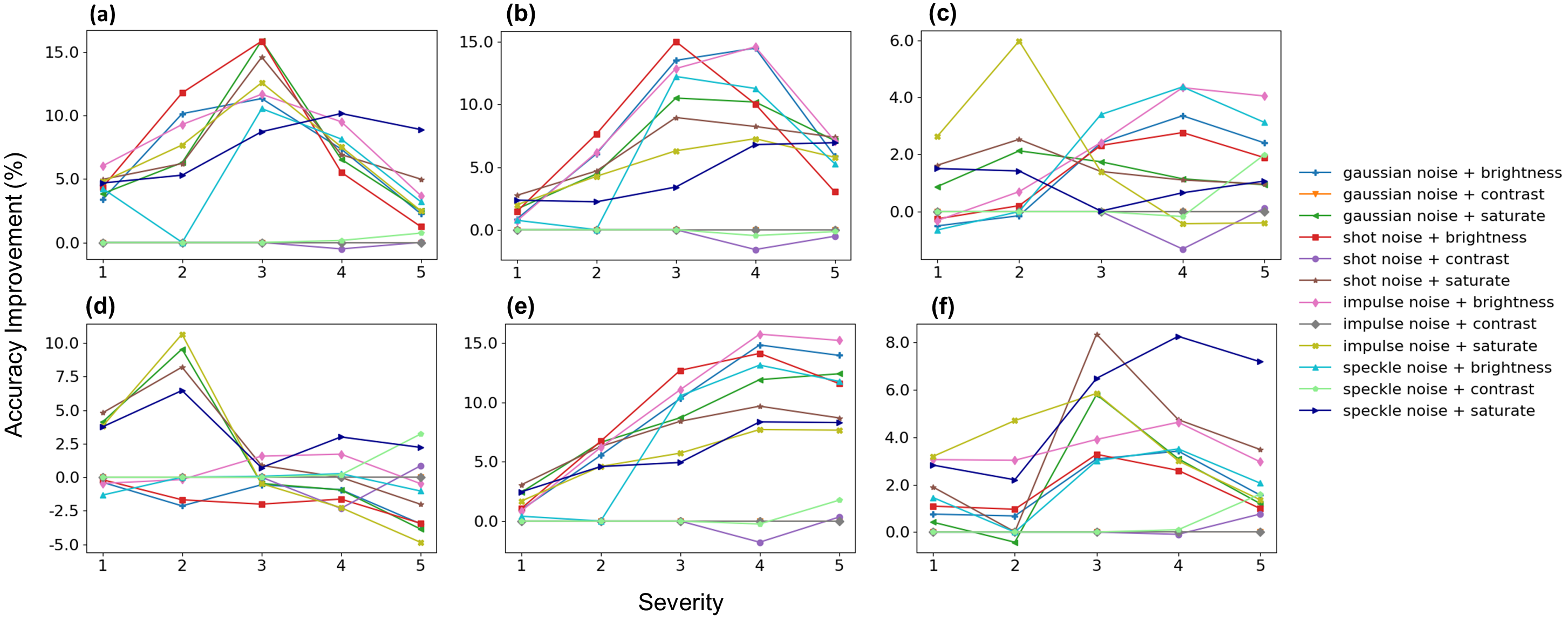}
    \caption{\small Accuracy improvements on composite ImageNet-C shifts for increasing severity evaluated on ImageNet classifiers trained with a) no data augmentation, b) AugMix, c) NoisyMix, d) DeepAugment, e) DeepAugment and AugMix, and f) PuzzleMix.}
\end{figure}

\vfill

\newpage
\section{Experimental Details for CIFAR-100-C}
\label{sec:cifar_details}
\input{supp_sections/cifar_details}
\vfill

\newpage
\section{Full Experimental Results for \ours{} on CIFAR-100-C Shifts}
\label{sec:cifar100c}
\input{tables/supp_cifar_c_0-972}
\begin{figure}[ht]
\centering\includegraphics[width=0.7\linewidth]{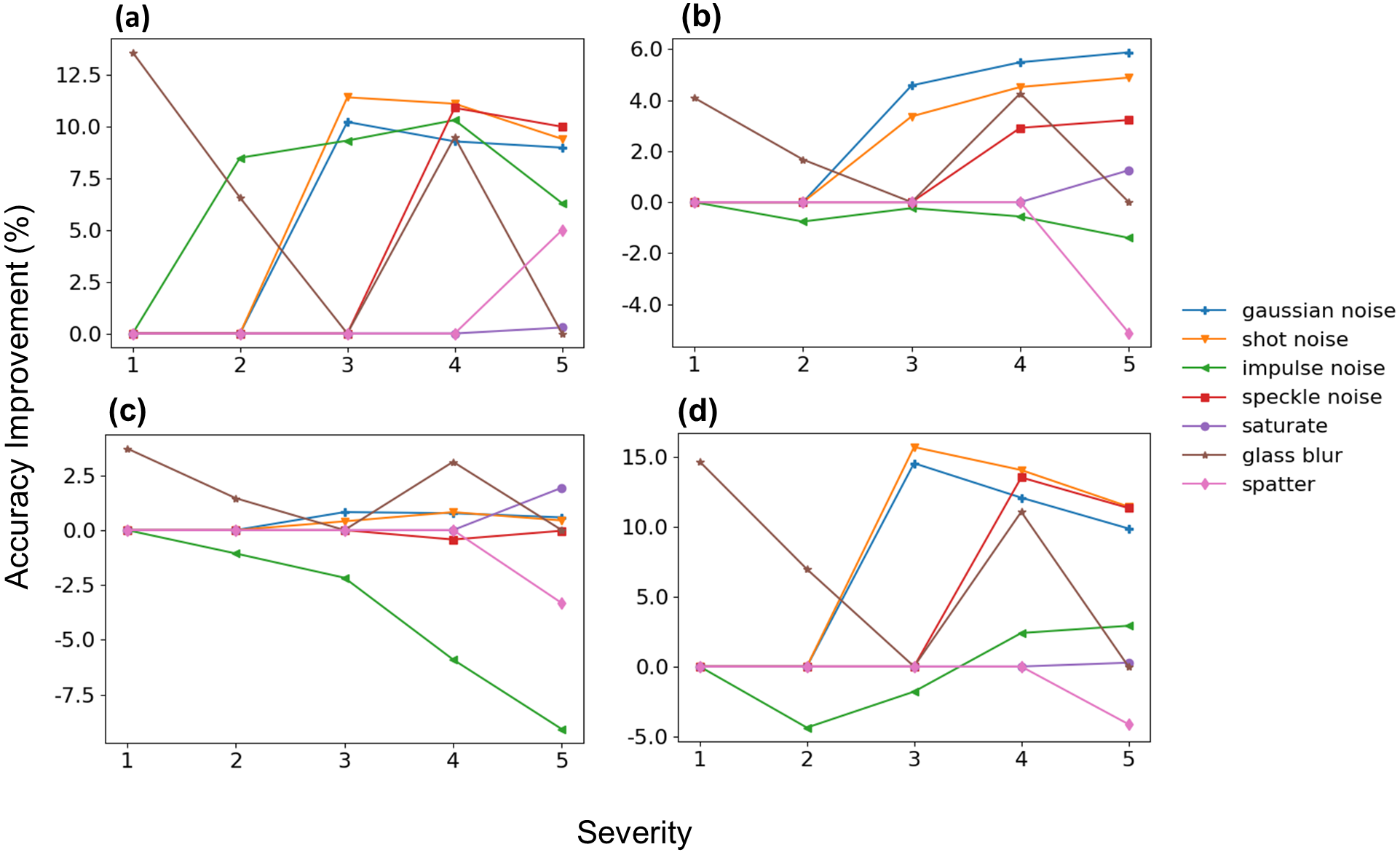}
    \caption{\small Accuracy improvements on CIFAR-100-C for increasing severity evaluated on CIFAR-100 classifiers trained with a) no data augmentation, b) AugMix, c) NoisyMix, d) PuzzleMix.}
\end{figure}
\vfill

\newpage
\section{Effect of Surrogate Shifts on \ours{} Performance}
\label{sec:clustering}
\input{supp_sections/clustering}
\vfill

\newpage
\section{Limitations of \ours{}}
\label{sec:limitations}
\input{supp_sections/limitations}
\vfill

%% file: supp_sections/distance_functions.tex
 Although we employ the Wasserstein distance to estimate the distance between a clean validation set of images and a corrupted set, a variety of alternatives can be used instead. For example, another natural choice for measuring the distance between two distributions is an $f$-divergence, such as the popular Total Variation (TV) distance, Kullback-Leibler (KL) divergence, and Jensen-Shannon (JS) divergence. In these cases, orthonormal projection can again be applied for dimensionality reduction, as equivalents to Proposition \ref{cor_lemma1} hold for $f$-divergences~\citep{lim2}. We note, however, that the TV distance, KL divergence, and JS divergence saturate when comparing distributions with disjoint supports, leading to potentially uninformative rewards when training the policy network. More reading on the TV distance, KL divergence, and JS divergence compared to the Wasserstein distance can be found in~\citet{arjovsky2017wasserstein}.

The previously mentioned methods for estimating the distance between two sets of images rely solely on the data itself, agnostic of the model performing inference on said data. While such distance functions lend flexibility across models and modalities, they fail to capitalize on information the model itself can provide. An alternative "distance function" that exploits the model is Agreement-on-the-Line (ALine-D), an approach for estimating the accuracy of a classifier on out-of-distribution data~\citep{baek2022agreement}. ALine-D builds on the empirical observation that many models exhibit strong linear correlation between in-distribution (ID) accuracy and out-of-distribution (OOD) accuracy~\citep{miller2021accuracy}. Combined with a similar phenomenon for the agreement among an ensemble of models on ID and OOD data, ALine-D estimates the accuracy of any ensemble member on OOD data, given only accuracies on ID data and agreements on ID/OOD data. This estimate can in turn be used to devise a distance function for use in our reward function, although incurring an additional computational burden from the evaluation of the ensemble. Ultimately, despite the benefits of ALine-D, we prioritize generalizability and instead select the Wasserstein distance.

%% file: supp_sections/orth_projection_intuition.tex
As discussed in the main text, we show that orthonormal projection is a reduction technique that preserves characteristics of the sample data at a distributional level. An intuitive motivation for this arises from a metric we refer to as the \textit{Cai-Lim distance}~\citep{lim2}.

We assume the same preliminaries as in Section~\ref{sec:orthonormal_proj} of the main text. Now denote by
\begin{align*}
    \Phi^-(\mu,d):=\{&\beta\in M(\mathbb{R}^m): \varphi_{V,b}(\mu)=\beta,\\
    &\text{for some } V\in O(d,n), b\in \mathbb{R}^m\}
\end{align*}
the set of Cai-Lim projections of $\mu$ onto $\mathbb{R}^m$. We call \textit{Cai-Lim distance} between $\mu\in M^p(\mathbb{R}^n)$ and $\nu\in M^p(\mathbb{R}^m)$, $m\leq n$, the smallest $p$-Wasserstein distance between $\nu$ and a Cai-Lim projection of $\mu$ onto $\mathbb{R}^m$,  for some $p\in [1,\infty]$. That is, 
\begin{equation}
    \label{eq:cailim_dist}
    W_p^{CL}(\mu, \nu) :=\inf_{\beta\in\Phi^-(\mu,d)} W_p \left( \beta,\nu \right).
\end{equation}
We note that \textit{Cai-Lim distance} is closely related to Monge and Kantorovich’s formulations of the optimal transport problem (subsection~\ref{sec:monge} below). A generalization of the Wasserstein distance, the Cai-Lim distance allows for measurements of distance between distributions of different dimensions. It follows directly from ~\eqref{eq:cailim_dist} that when an orthonormal projection is applied, $\nu = \varphi_{V,0}(\mu)$, the Cai-Lim distance is trivially zero, $W_p^{CL}(\mu, \nu)=0$. This suggests that the orthonormal projection preserves information about the distance between two distributions.

\subsection{Cai-Lim Distance and Optimal Transport} \label{sec:monge}
Monge's formulation of the optimal transport problem can be formulated as follows. Let $\mathcal{X},\mathcal{Y}$ be two separable metric Radon spaces. Let $c:\mathcal{X}\times\mathcal{Y} \rightarrow [0,\infty]$ be a Borel-measurable function. Given probability measures $\mu$ on $\mathcal{X}$ and $\nu$ on $\mathcal{Y}$, Monge's formulation of the optimal transportation problem is to find a transport map $T:\mathcal{X}\rightarrow\mathcal{Y}$ that realizes the infimum
\begin{equation}\label{eq1}
    \inf_{T_\star(\mu)=\nu} \int_\mathcal{X} c(x,T(x)) \text{ d}\mu(x),
\end{equation}
where $T_\star(\mu)\equiv \mu\circ T^{-1}$ is the pushforward of $\mu$ by $T$. 

Monge's formulation of the optimal transportation problem can be ill-posed, because sometimes there is no $T$ satisfying $T_\star(\mu)=\nu$. Equation \eqref{eq1} can be improved by adopting Kantorovich's formulation of the optimal transportation problem, which is to find a probability measure $\gamma\in\Gamma(\mu,\nu)$ that attains the infimum
\begin{equation}\label{eq2}
    \inf_{\gamma\in\Gamma(\mu,\nu)} \int_{\mathcal{X}\times\mathcal{Y}} c(x,y) \text{ d}\gamma(x,y),
\end{equation}
where $\Gamma(\mu,\nu)$ is the set of joint probability measures on $\mathcal{X}\times\mathcal{Y}$ whose marginals are $
\mu$ on $\mathcal{X}$ and $\nu$ on $\mathcal{Y}$. A minimizer for this problem always exists when the cost function $c$ is lower semi-continuous and $\Gamma(\mu,\nu)$ is a tight collection of measures.

Our Cai-Lim distance formulation, Equation (\textcolor{red}{5}), can be written as 
\begin{equation}\label{eq3}
    \inf_{\beta\in\Phi^-(\mu,d)} \left[ \left( \inf_{\gamma\in\Gamma(\beta,\nu)} \int_{\mathbb{R}^{2d}} \|x-y\|_2^p \text{ d}\gamma(x,y) \right)^{1/p} \right].
\end{equation}
Therefore, \eqref{eq3} is a combination of \eqref{eq1} and \eqref{eq2}. To see this, notice that $\mathcal{X}=\mathcal{Y}=\mathbb{R}^d$, and $(x,y) \mapsto c(x,y)=\|x-y\|_2^p$, for some $p\in [1,\infty]$. Hence, the inner part of \eqref{eq3} corresponds to Kantorovich's formulation of the optimal transportation problem. The outer $\inf$, instead, has a Mongenian flavor to it. With this, we mean that the probability measure $\beta$ must be the one minimizing the $p$-Wasserstein distance between $\nu$ and all the elements in the set $\Phi^-(\mu,d)$ of Cai-Lim projections of $\mu$ onto $\mathbb{R}^d$, $d\leq n$. Because
$$\Phi^-(\mu,d):=\{\beta\in M(\mathbb{R}^m): \varphi_{V,b}(\mu)=\beta, \text{ for some } V\in O(d,n), b\in \mathbb{R}^m\},$$
and $\varphi_{V,b}(\mu) \equiv {\varphi_{V,b}^{-1}}_\star(\mu):=\mu\circ\varphi_{V,b}^{-1}$ is the pushforward of $\mu$ by $\varphi_{V,b}^{-1}$, this reminds us (heuristically) of Monge's formulation. The fact that \eqref{eq3} has a solution is guaranteed by \citet[Section II]{lim2}.

Some further references can be found at \citet{kantorovich,villani}.

%% file: supp_sections/related_work.tex
\textbf{Offline approaches for classification robustness.}
Most common among offline approaches are data augmentation techniques, which expand the experiences that a model would be exposed to during training. AugMix~\citep{hendrycks2019augmix} applies a set of transformations (e.g., rotation, translation, shear, etc.) to distort the original image. NoisyMix \citep{noisy_mix_paper} further injects noises into the augmented examples. ManifoldMix \citep{manifold_mix} also uses a data augmentation approach, but focuses on generating examples from a lower-dimensional latent manifold. Other notable work in the recent literature includes CutMix \citep{cut_mix}, PuzzleMix \citep{puzzle_mix}, and DeepAugment \citep{deep_augment}.
Another offline approach is domain adaptation, in which models are trained on data from one distribution (i.e., source data) are adapted to perform well on unlabeled data in some other distribution (i.e., target data) \citep{ben2010theory}. Among the many domain adaptation techniques, input-level translation is most similar to our work. These techniques attempt to move the source distribution closer to the target distribution by augmenting the source data in the input or feature space \citep{bousmalis2017unsupervised, hoffman2018cycada}. Finally, a related but distinct offline technique for classification robustness is adversarial training~\citep{goodfellow2014explaining}, which aims to improve robustness to adversarial examples rather than distribution shift. In contrast to these offline approaches, our online technique allows for greater flexibility to distribution shifts unforeseen at train time. However, our method can be used in concert with offline approaches to further improve performance on corrupted samples.

\textbf{Online approaches for classification robustness.}
A broad literature adjust for distribution shift at inference time. Test-time adaptation methods, for example,  perform additional training at inference time to adapt models to incoming unlabeled data \cite{gandelsman2022test, wang2022continual, mummadi2021test}. A major challenge of such approaches is the generation of reliable pseudo-labels, especially for data under distribution shift.

Other online approaches, including ours, require no additional training of the model. Test time augmentation methods, for example, apply a set of transformations separately to the input, then ensemble the model's predictions on the different transformations of the data \citep{simonyan2014very, krizhevsky2012imagenet, he2016deep, guo2017countering, mummadi2021test}. In addition to the added complexity, a key disadvantage of this method is that the same transformations are applied to every input seen online. To overcome this, other techniques select the single best transformation among several candidates based on the model's performance on these candidates \citep{kim2020learning, lyzhov2020greedy}. Our method similarly selects transformations online and dynamically, but can do so without querying the model. Hence, when transformations are selected, they can be applied to any model to improve performance. Furthermore, our work uniquely identifies MDPs as an equivalent formulation of the transformation selection problem, allowing for the application of standard reinforcement learning techniques.

\textbf{Distribution shift detection.}
Out-of-distribution shift detection has gained significant interest as more machine learning models are deployed into safety critical systems \citep{kaur_codit_2022,hendrycks_baseline_2018,lee_training_2018}. 
In \citep{rabanser_failing_2019}, various approaches to dimensionality reduction are explored and empirical results demonstrating the benefit to distribution shift detection are presented. 
Our work expands on these observations and utilizes a theoretically sound approach of using orthogonal projections for dimensionality reduction.

\textbf{Deep learning for image restoration.}
Deep learning approaches for image restoration and recovery from corruption are well studied in the literature and exhaustive review is beyond the scope of this paper. 
Most of these approaches focus on a specific corruption type or a set of corruptions. 
For example, \citet{zhang_beyond_2017} approaches Gaussian denoising by constructing denoising convolutional feed-forward networks. 
For a recent survey of deep blurring approaches and comprehensive approaches targeting multiple sets of noises we refer to \citet{zhang_deep_2022,ledig_photo-realistic_2017,zamir_multi-stage_2021,vedaldi_learning_2020,zamir_learning_2022}.
Such approaches utilize a generative model to restore information lost due to a low resolution. 
These methods are limited in recovery and can result in image artifacts which degrade classification performance. 
Here, we attempt to alleviate such limitations by using a composition of transformations which are selected automatically depending on the context.
Additionally, the performance can be enhanced with additional image transformations in the action library (see Section \ref{sec:policy_net}).

%% file: supp_sections/experiments.tex
\subsection{Surrogate Corruptions}
Since in practice the corrupting transformations present at test time are unknown to the designer, we select surrogate corruptions for training the operability classifier and the policy network. To maximize generalization, we select six surrogate corruptions which subject the images to a broad set of transformations and influence each dimension of their state representations. Each serves to increase or decrease one of entropy, brightness, or standard deviation. Below, we describe each of the corruptions in greater detail.


\textbf{Uniform Noise.}
We use additive uniform noise, $Uniform(a, b)$, as a representative of shifts that increase the entropy of the image. We select $b = -a \in \{0.14, 0.22, 0.32, 0.40, 0.90\}$ to generate five severity levels of uniform noise shift.

\textbf{Median Blur.}
We use median blur as a representative of shifts that decrease the entropy of the image. We select a square kernel size $k \in \{2, 3, 4, 5, 6\}$ to generate five severity levels of median blur shift.

\textbf{$\gamma$ Correction (Type A).}
We adjust the $\gamma$ value of the images, with $\gamma > 1$, to create a representative of shifts that increase the average brightness of the image. We select $\gamma \in \{1.4, 1.7, 2.0, 2.5, 3.0\}$ to generate five severity levels of this $\gamma$ correction.

\textbf{$\gamma$ Correction (Type B).}
We adjust the $\gamma$ value of the images, with $\gamma < 1$, to create a representative of shifts that decrease the average brightness of the image. We select $\gamma \in \{0.9, 0.8, 0.7, 0.6, 0.5\}$ to generate five severity levels of this $\gamma$ correction. 

\textbf{Sigmoid Correction (Type A).}
We perform sigmoid correction to create a representative of shifts that increase the standard deviation of the image. We select cutoff 0.5 and gain $g \in \{7, 8, 9, 10, 11\}$ to generate five severity levels of this sigmoid correction. 

\textbf{Sigmoid Correction (Type B).}
We perform sigmoid correction to create a representative of shifts that decrease the standard deviation of the image. We select cutoff 0.5 and gain $g \in \{7, 6, 5, 4, 3\}$ to generate five severity levels of this sigmoid correction. 

Figure~\ref{fig:sample_ims} shows samples of images under these surrogate corruptions. We now describe the criteria used to select the parameters of each corruption listed above for five severity levels. Given a random sample $\mathbf{V}_c$ of the surrogate corrupted data and a random sample $\mathbf{V}$ of ImageNet data, we select the parameter for severity level 5 (strongest) such that
\begin{equation}
    \left\vert \frac{\mathbb{F}_R(\mathbf{V}_c)_i - \mathbb{F}_R(\mathbf{V})_i}{\mathbb{F}_R(\mathbf{V})_i} \right\vert \in [0.3, 1.0],
\end{equation}
where $\mathbb{F}_R(\mathbf{X})_i$ denotes the $i$th index of $\mathbb{F}_R(\mathbf{X})$. We use $i=2$ for uniform noise and median blur, $i=0$ for $\gamma$ correction type A/B, and $i=1$ for sigmoid correction type A/B. For the severity levels 1-4, we select parameters such that there is large variation in $\mathbb{F}_R(\mathbf{X})_i$ across the five severity levels. 
\begin{figure}[ht]
    \centering
    \includegraphics[width=\textwidth]{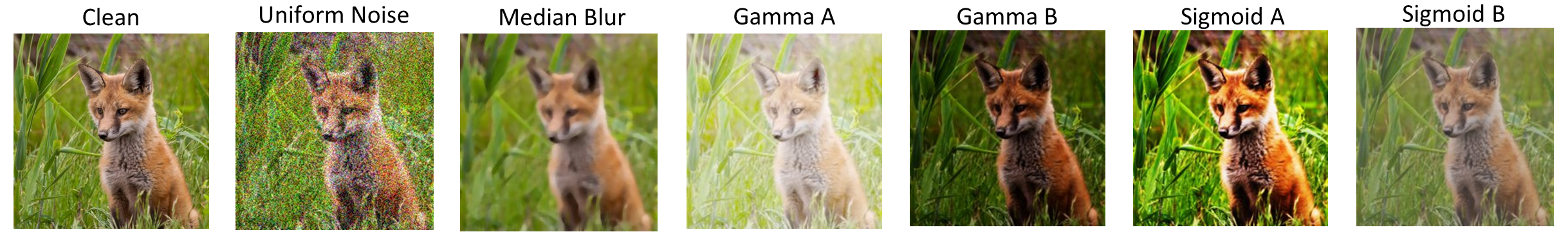}
    \caption{\small Sample images from ImageNet subjected to surrogate corruptions.}
    \label{fig:sample_ims}
\end{figure}

\newpage
\subsection{Operability Classifier.} 
On the surrogate corruptions, we train a binary decision tree with depth of 8 selected by a hyperparameter sweep, and using  Gini index as the split criterion. For label generation, we use a ResNet-50 classifier trained without data augmentation. The resulting classifier has an AUROC of 0.80 on a held out set of surrogate corrupted data. Of the six surrogate shifts, the operability classifier usually picks data under median blur data as inoperable and uniform noise data as operable. A full evaluation of the operability classifier is shown in Table~\ref{tab:imagenet_dt}.

\input{tables/imagenet_decision_tree}

\subsection{Policy Network.} Similar to the operability classifier, the policy network must be trained on surrogate corruptions selected at design time. However, when selecting surrogates for the policy network, we must additionally limit the size of the state space that must be explored to avoid convergence issues. Thus, we select a subset of the six surrogate shifts previously described, consisting of two shifts with low inoperability rates, uniform noise and $\gamma$ correction with $\gamma > 1$.

We manually select an action library $\mathbb{A}$ of eight correcting transforms, which aim to address both of these surrogate shifts. For each transform, we select weak parameters to guard against single actions with irreversibly destructive effect. Simultaneously, this allows \ours{} to adapt accordingly to the noise severity. To target uniform noise, we include several denoising actions. These are a denoising convolutional neural network from the MATLAB Deep Learning Toolbox, a bilateral filter with filter size 2, wavelet denoising with BayesShrink thresholding~\citep{chang2000adaptive}, and wavelet denoising with VisuShrink thresholding~\citep{donoho1994ideal}. To target $\gamma$ correction, we include three Contrast Limited Adaptive Histogram Equalization (CLAHE) actions. Broadly, CLAHE applies histogram equalization over small tiles in the image, while limiting the allowable amount of contrast change. Our selected three CLAHE transformations, are CLAHE with (tile size, limit) of (2,1), (2,2), and (6,1). Finally, we include in $\mathbb{A}$ an action that does not enact any transformation, allowing for inaction in the presence of benign corruptions.

We train a deep neural network policy $\pi$ with actions $\mathbb{A}$ and surrogate corruptions defined above. Sequences are limited to five actions, but a shorter sequence may occur based on the stopping criteria in Algorithm~\ref{alg:dsr}. We choose our stopping condition threshold to be $\alpha = 0.9$ and $\beta = 0.995$. For the estimation of Wasserstein distances, we convert images to grayscale and project them to 5000 dimensions using a randomly generated orthonormal matrix. We select the reward hyperparameters $\lambda = 20$ and $\omega=0.994$. The neural network policy $\pi$ is trained using the advantage actor critic  algorithm. For both the actor and critic network, we employ a linear ReLU network with two hidden layers of 128 and 256 units. We train the actor and critic networks until convergence (about 1600 episodes) with learning rate $10^{-4}$, learning frequency 1, discount factor 0.9, and exploration rate decaying exponentially from 0.9 to 0.1 according the rule $0.9^{0.07*(episode+1)}$.

The training curves are shown in Figure~\ref{fig:training_curves}.
\begin{figure}[ht]
    \centering
    \includegraphics[width=0.5\textwidth]{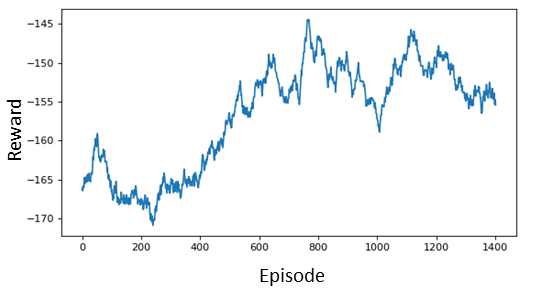}
    \caption{\small Reinforcement learning training curves, averaged with a moving window of 200 episodes. Values shown are the reward at episode termination.}
    \label{fig:training_curves}
\end{figure}

%% file: tables/imagenet_decision_tree.tex
\begin{table*}[ht]\scriptsize 
\caption{\small Percentage of images that the operability classifier labels as inoperable for surrogate and ImageNet-C shifts. Shifts gamma type A and gamma type B refer to $\gamma$ correction with $\gamma>1$ and $\gamma<1$, respectively. Shifts sigmoid type A and sigmoid type B refer to sigmoid correction with parameters that increase and decrease the standard deviation of the image, respectively. In general, the classifier labels blur-type shifts as inoperable most often. Conversely, the classifier labels noise-type shifts as operable most often.}
\label{tab:imagenet_dt}
\setlength{\tabcolsep}{2pt}
\begin{center}
\begin{tabular}{|p{22ex}|p{18ex}| p{3ex} |p{22ex}|p{18ex}|}
    \cline{1-2} \cline{4-5}
    \multicolumn{2}{|c|}{\textbf{Surrogates}} & & \multicolumn{2}{|c|}{\textbf{ImageNet-C}}\\
    \cline{1-2} \cline{4-5}
    \textbf{shift} & \textbf{\% inoperable} & & \textbf{shift} & \textbf{\% inoperable} \\
    \cline{1-2} \cline{4-5}
    uniform noise & 0.01 & & gaussian noise &
    0.004\\
    \cline{1-2} \cline{4-5}
    median blur & 48.69 & & shot noise & 0.02\\
    \cline{1-2} \cline{4-5}
    gamma type A & 8.72 & & impulse noise & 0.009\\
    \cline{1-2} \cline{4-5}
    gamma type B & 12.56 & & defocus blur & 80.35\\
    \cline{1-2} \cline{4-5}
    sigmoid type A & 7.10 & & glass blur & 70.22 \\
    \cline{1-2} \cline{4-5}
    sigmoid type B & 15.57 & & motion blur & 63.86\\
    \cline{1-2} \cline{4-5}
    \multicolumn{1}{c}{} & \multicolumn{1}{c}{} & & zoom blur & 66.76 \\
    \cline{4-5}
    \multicolumn{1}{c}{} & \multicolumn{1}{c}{} & & snow & 0.18\\
    \cline{4-5}
    \multicolumn{1}{c}{} & \multicolumn{1}{c}{} & & frost & 0.23\\
    \cline{4-5}
    \multicolumn{1}{c}{} & \multicolumn{1}{c}{} & & fog & 30.18\\
    \cline{4-5}
    \multicolumn{1}{c}{} & \multicolumn{1}{c}{} & & brightness & 6.99\\
    \cline{4-5}
    \multicolumn{1}{c}{} & \multicolumn{1}{c}{} & & contrast & 1.77\\
    \cline{4-5}
    \multicolumn{1}{c}{} & \multicolumn{1}{c}{} & & elastic & 23.09\\
    \cline{4-5}
    \multicolumn{1}{c}{} & \multicolumn{1}{c}{} & & pixelate & 43.83\\
    \cline{4-5}
    \multicolumn{1}{c}{} & \multicolumn{1}{c}{} & & jpeg & 41.63\\
    \cline{4-5}
    \multicolumn{1}{c}{} & \multicolumn{1}{c}{} & & speckle noise & 0.03\\
    \cline{4-5}
    \multicolumn{1}{c}{} & \multicolumn{1}{c}{} & & gaussian blur & 76.21\\
    \cline{4-5}
    \multicolumn{1}{c}{} & \multicolumn{1}{c}{} & & spatter & 2.26\\
    \cline{4-5}
    \multicolumn{1}{c}{} & \multicolumn{1}{c}{} & & saturate & 10.31 \\
    \cline{4-5}
\end{tabular}
\end{center}
\end{table*}

%% file: tables/supp_imagenet_c.tex
\begin{table*}[ht]\scriptsize 
\caption{\small Average accuracies (\%) on each ImageNet-C shift with and without \ours for ResNet-50 classifiers. Accuracy improvement is denoted by $\Delta = \text{R (recovered)} - \text{S (shifted)}$.
Values are over 5 severity levels with 3 trials each.}
\label{tab:supp_imagenet_cp_results}
\setlength{\tabcolsep}{2pt}
\begin{center}
\tiny
\begin{tabular}{|p{13ex}|p{6.5ex} p{6.5ex} p{7ex}|p{6.5ex} p{6.5ex} p{7ex}|p{6.5ex} p{6.5ex} p{7ex}|p{6.5ex} p{6.5ex} p{7ex}|p{6.5ex} p{6.5ex} p{7ex}|p{6.5ex} p{6.5ex} p{7ex}|}
    \cline{2-19}
    \multicolumn{1}{c}{} & \multicolumn{3}{|c|}{No Data Aug} & \multicolumn{3}{|c|}{AugMix} & \multicolumn{3}{|c|}{NoisyMix} & \multicolumn{3}{|c|}{DeepAugment} & \multicolumn{3}{|c|}{DeepAug+AugMix} & \multicolumn{3}{|c|}{PuzzleMix}\\
    \hline
    shift &
    S & R & $\Delta$ &
    S & R & $\Delta$ &
    S & R & $\Delta$ &
    S & R & $\Delta$ &
    S & R & $\Delta$ &
    S & R & $\Delta$\\
    \hline
    none &
    74.52 & 74.52 & 0.00 &
    75.94 & 75.94 & 0.00 &
    76.22 & 76.22 & 0.00 &
    75.86 & 75.86 & 0.00 &
    75.26 & 75.26 & 0.00 &
    75.63 & 75.63 & 0.00\\
    \hline
    gaussian noise &
    31.11 & 43.44 & \textbf{12.33} &
    41.90 & 50.87 & \textbf{8.98} &
    52.71 & 55.51 & \textbf{2.80} &
    59.07 & 59.48 & \textbf{0.41} &
    55.39 & 61.43 & \textbf{6.05} &
    41.48 & 46.94 & \textbf{5.46}\\
    \hline
    shot noise &
    28.61 & 42.81 & \textbf{14.21} &
    41.78 & 50.93 & \textbf{9.15} &
    51.81 & 55.38 & \textbf{3.57} &
    58.21 & 58.46 & \textbf{1.24} &
    55.76 & 62.37 & \textbf{6.61} &
    37.39 & 45.56 & \textbf{7.77}\\
    \hline
    impulse noise &
    26.57 & 39.36 & \textbf{12.79} &
    38.78 & 47.49 & \textbf{8.71} &
    50.73 & 53.37 & \textbf{2.64} &
    58.61 & 58.38 & -0.23 &
    55.16 & 60.67 & \textbf{5.50} &
    35.28 & 42.82 & \textbf{7.54}\\
    \hline
    defocus blur &
    35.21 & 35.21 & 0.00 &
    44.48 & 44.48 & 0.00 &
    44.72 & 44.72 & 0.00 &
    48.10 & 48.10 & 0.00 &
    55.52 & 55.52 & 0.00 &
    38.02 & 38.02 & 0.00\\
    \hline
    glass blur &
    25.55 & 25.55 & 0.00 &
    32.97 & 32.97 & 0.00 &
    35.71 & 35.71 & 0.00 &
    38.39 & 38.39 & 0.00 &
    44.45 & 44.45 & 0.00 &
    25.71 & 25.71 & 0.00\\
    \hline
    motion blur &
    36.25 & 36.25 & 0.00 &
    49.63 & 49.63 & 0.00 &
    49.53 & 49.53 & 0.00 &
    45.46 & 45.46 & 0.00 &
    57.56 & 57.56 & 0.00 &
    39.29 & 39.29 & 0.00\\
    \hline
    zoom blur &
    36.27 & 36.27 & 0.00 &
    47.45 & 47.45 & 0.00 &
    47.16 & 47.16 & 0.00 &
    39.84 & 39.84 & 0.00 &
    50.54 & 50.54 & 0.00 &
    39.85 & 39.85 & 0.00\\
    \hline
    snow &
    30.51 & 29.28 & -1.13 &
    37.89 & 36.85 & -1.04 &
    43.20 & 41.52 & -1.68 &
    41.71 & 39.91 & -1.80 &
    47.68 & 46.36 & -1.32 &
    39.48 & 37.74 & -1.75\\
    \hline
    frost &
    35.16 & 35.07 & -0.09 &
    41.39 & 41.24 & -0.15 &
    50.05 & 49.46 & -0.59 &
    46.87 & 46.08 & -0.78 &
    51.21 & 50.26 & -0.95 &
    46.96 & 45.75 & -1.21\\
    \hline
    fog &
    42.79 & 42.79 & 0.00 &
    44.97 & 44.97 & 0.00 &
    51.34 & 51.34 & 0.00 &
    49.88 & 49.88 & 0.00 &
    54.46 & 54.46 & 0.00 &
    55.62 & 55.62 & 0.00\\
    \hline
    brightness &
    65.17 & 65.72 & \textbf{0.55} &
    67.35 & 68.46 & \textbf{1.11} &
    68.82 & 69.87 & \textbf{1.05} &
    69.04 & 69.73 & \textbf{0.69} &
    69.42 & 70.18 & \textbf{0.76} &
    69.59 & 69.67 & \textbf{0.08}\\
    \hline
    contrast &
    35.56 & 37.69 & \textbf{2.14} &
    48.96 & 49.85 & \textbf{0.89} &
    50.37 & 52.74 & \textbf{2.37} &
    44.89 & 48.23 & \textbf{3.33} &
    56.01 & 57.40 & \textbf{1.39} &
    50.56 & 52.87 & \textbf{2.30}\\
    \hline
    elastic &
    43.24 & 43.24 & 0.00 &
    50.18 & 50.18 & 0.00 &
    51.02 & 51.02 & 0.00 &
    50.35 & 50.35 & 0.00 &
    53.26 & 53.26 & 0.00 &
    43.31 & 43.31 & 0.00\\
    \hline
    pixelate &
    45.51 & 45.51 & 0.00 &
    57.25 & 57.25 & 0.00 &
    54.23 & 54.23 & 0.00 &
    64.31 & 64.31 & 0.00 &
    67.31 & 67.31 & 0.00 &
    49.03 & 49.03 & 0.00\\
    \hline
    jpeg &
    52.47 & 52.47 & 0.00 &
    58.49 & 58.49 & 0.00 &
    61.85 & 61.85 & 0.00 &
    56.99 & 56.99 & 0.00 &
    61.31 & 61.31 & 0.00 &
    56.83 & 56.83 & 0.00\\
    \hline
    speckle noise &
    36.09 & 49.26 & \textbf{13.18} &
    50.61 & 56.94 & \textbf{6.33} &
    57.67 & 60.88 & \textbf{3.22} &
    62.21 & 63.81 & \textbf{1.59} &
    60.93 & 65.66 & \textbf{4.74} &
    42.24 & 51.92 & \textbf{9.68}\\
    \hline
    gaussian blur &
    38.08 & 38.08 & 0.00 &
    47.17 & 47.17 & 0.00 &
    47.30 & 47.30 & 0.00 &
    51.93 & 51.93 & 0.00 &
    57.53 & 57.53 & 0.00 &
    41.07 & 41.07 & 0.00\\
    \hline
    spatter &
    46.65 & 46.44 & -0.20 &
    53.25 & 52.93 & -0.32 &
    57.63 & 57.34 & -0.29 &
    53.74 & 53.58 & -0.16 &
    57.75 & 57.61 & -0.14 &
    53.27 & 52.95 & -0.32\\
    \hline
    saturate &
    59.00 & 59.17 & \textbf{0.17} &
    61.42 & 61.89 & \textbf{0.47} &
    63.48 & 64.00 & \textbf{0.52} &
    64.59 & 64.81 & \textbf{0.22} &
    65.79 & 66.12 & \textbf{0.33} &
    65.96 & 65.60 & -0.37\\
    \hline
\end{tabular}
\end{center}
\end{table*}


%% file: supp_sections/inaction.tex
As noted in the main text, \ours{} refrains from taking action when no shift is present and for a variety of ImageNet-C shifts. This inaction occurs due to a combination of both our operability classifier and our stopping conditions in Algorithm~\ref{alg:dsr}. For example, blur-type shifts are frequently labeled as inoperable: over 80\% of defocus blur images are inoperable, and similarly for other blur-type shifts (see Appendix~\ref{sec:experiments}). We also observe that inaction can at times be overly conservative. An example case is fog shift, where an appropriate action exists (e.g., CLAHE), yet no action is taken. We now explore this inaction on fog shift further.

Although \ours{} chooses to take no action on fog shift, there exists an appropriate action for this shift. In fact, Contrast Limited Adaptive Histogram Equalization (CLAHE) can achieve a 4.15\% accuracy improvement, averaged across all five severity levels, on the AugMix classifier. Further, if allowed to take the full 5-step sequence of actions, \ours{} would initially select a CLAHE action to correct for fog shift. This is supported by Figure~\ref{fig:kmeans_fog}, which shows that samples from fog shift have a state representation similar to those from our surrogate $\gamma$ correction (with $\gamma > 1$) shift and other ImageNet-C shifts for which CLAHE is a strong correction (i.e., brightness, contrast, and saturate), leading \ours{} to select a similar action for fog shift. However, \ours{} uses stopping conditions that can prematurely terminate its procedure. The nature of fog corruption is such that, after taking the first action, \ours{} decides to terminate the process. We speculate that this can be improved by using larger batch sizes for the validation set, which would allow for a more accurate estimate of the Wasserstein distance. Simultaneously, this would come at a higher computational cost for training the RL policy.

\begin{figure*}[ht]
    \centering
        \includegraphics[width=0.6\textwidth]{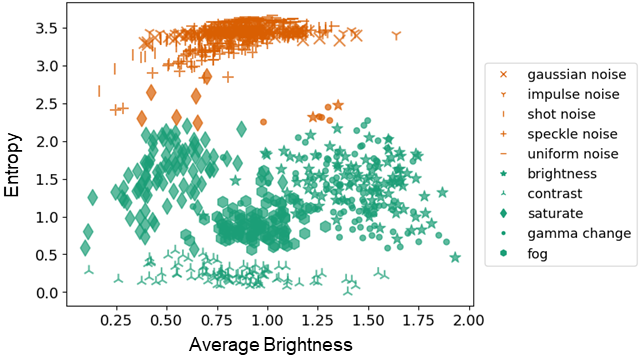}
    \caption{\small $K$-means clustering of sampled state representations for brightness, contrast, and saturate shifts from ImageNet-C, fog shift from ImageNet-C, and our surrogates. 100 samples were taken from each shift. $K=2$ clusters were chosen by selecting $K$ from the range $[1,10]$ via the elbow method. The projection onto the entropy and average brightness dimensions of the state representation is shown. Samples from fog shift are clustered with those from the $\gamma$ correction, brightness, contrast, and saturate shifts.}
    \label{fig:kmeans_fog}
\end{figure*}

%% file: supp_sections/cifar_details.tex
For the CIFAR-100-C benchmark, we retrain only the operability classifier. We train on the same six surrogates, regenerated for CIFAR-100. For label generation, we use a Wide ResNet with depth 28 and widening factor 10 (WRN-28-10) trained on CIFAR-100 without data augmentation~\citep{zagoruyko2016wide}. The resulting operability classifier has an AUROC of 0.73 on a held out set of surrogate data. Predictions made by the operability classifier on surrogate and CIFAR-100-C shifts are shown in Table~\ref{tab:cifar_dt}. Lastly, we update our stopping condition parameters to $\alpha = 0.9$, $\beta = 0.972$.

For evaluation, we reuse the WRN-28-10. We also evaluate on WRNs trained with AugMix (WRN-40-2), NoisyMix (WRN-28-4), and PuzzleMix (WRN-28-10).  

\input{tables/cifar_decision_tree}

%% file: tables/cifar_decision_tree.tex
\begin{table*}[ht]\scriptsize 
\caption{\small Percentage of images that the operability classifier labels as inoperable for surrogate and CIFAR-100-C shifts. Shifts gamma type A and gamma type B refer to $\gamma$ correction with $\gamma>1$ and $\gamma<1$, respectively. Shifts sigmoid type A and sigmoid type B refer to sigmoid correction with parameters that increase and decrease the standard deviation of the image, respectively. In general, the classifier labels blur-type shifts as (with the exception of glass blur) inoperable most often. Conversely, the classifier labels noise-type shifts as operable most often.}
\label{tab:cifar_dt}
\setlength{\tabcolsep}{2pt}
\begin{center}
\begin{tabular}{|p{22ex}|p{18ex}| p{3ex} |p{22ex}|p{18ex}|}
    \cline{1-2} \cline{4-5}
    \multicolumn{2}{|c|}{\textbf{Surrogates}} & & \multicolumn{2}{|c|}{\textbf{CIFAR-100-C}}\\
    \cline{1-2} \cline{4-5}
    \textbf{shift} & \textbf{\% inoperable} & & \textbf{shift} & \textbf{\% inoperable} \\
    \cline{1-2} \cline{4-5}
    uniform noise & 0.002 & & gaussian noise &
    0.004\\
    \cline{1-2} \cline{4-5}
    median blur & 43.88 & & shot noise & 0.03\\
    \cline{1-2} \cline{4-5}
    gamma type A & 6.72 & & impulse noise & 1.65\\
    \cline{1-2} \cline{4-5}
    gamma type B & 11.30 & & defocus blur & 44.34\\
    \cline{1-2} \cline{4-5}
    sigmoid type A & 7.05 & & glass blur & 2.59\\
    \cline{1-2} \cline{4-5}
    sigmoid type B & 15.09 & & motion blur & 37.62\\
    \cline{1-2} \cline{4-5}
    \multicolumn{1}{c}{} & \multicolumn{1}{c}{} & & zoom blur & 45.73\\
    \cline{4-5}
    \multicolumn{1}{c}{} & \multicolumn{1}{c}{} & & snow & 3.42\\
    \cline{4-5}
    \multicolumn{1}{c}{} & \multicolumn{1}{c}{} & & frost & 1.90\\
    \cline{4-5}
    \multicolumn{1}{c}{} & \multicolumn{1}{c}{} & & fog & 25.83\\
    \cline{4-5}
    \multicolumn{1}{c}{} & \multicolumn{1}{c}{} & & brightness & 9.41\\
    \cline{4-5}
    \multicolumn{1}{c}{} & \multicolumn{1}{c}{} & & contrast & 20.97\\
    \cline{4-5}
    \multicolumn{1}{c}{} & \multicolumn{1}{c}{} & & elastic & 24.73\\
    \cline{4-5}
    \multicolumn{1}{c}{} & \multicolumn{1}{c}{} & & pixelate & 26.80\\
    \cline{4-5}
    \multicolumn{1}{c}{} & \multicolumn{1}{c}{} & & jpeg & 15.39\\
    \cline{4-5}
    \multicolumn{1}{c}{} & \multicolumn{1}{c}{} & & speckle noise & 0.04\\
    \cline{4-5}
    \multicolumn{1}{c}{} & \multicolumn{1}{c}{} & & gaussian blur & 51.42\\
    \cline{4-5}
    \multicolumn{1}{c}{} & \multicolumn{1}{c}{} & & spatter & 4.70\\
    \cline{4-5}
    \multicolumn{1}{c}{} & \multicolumn{1}{c}{} & & saturate & 10.63\\
    \cline{4-5}
\end{tabular}
\end{center}
\end{table*}

%% file: tables/supp_cifar_c_0-972.tex
\begin{table*}[ht]\scriptsize 
\caption{\small Average accuracies (\%) on each CIFAR-100-C shift with and without \ours{} for Wide ResNet classifiers. Accuracy improvement is denoted by $\Delta = \text{R (recovered)} - \text{S (shifted)}$.
Values are over 5 severity levels with 3 trials each.}
\label{tab:supp_cifar_results}
\setlength{\tabcolsep}{2pt}
\begin{center}
\tiny
\begin{tabular}{|p{13ex}|p{6.5ex} p{6.5ex} p{7ex}|p{6.5ex} p{6.5ex} p{7ex}|p{6.5ex} p{6.5ex} p{7ex}|p{6.5ex} p{6.5ex} p{7ex}|}
    \cline{2-13}
    \multicolumn{1}{c}{} & \multicolumn{3}{|c|}{No Data Aug} & \multicolumn{3}{|c|}{AugMix} & \multicolumn{3}{|c|}{NoisyMix}  & \multicolumn{3}{|c|}{PuzzleMix}\\
    \hline
    shift &
    S & R & $\Delta$ &
    S & R & $\Delta$ &
    S & R & $\Delta$ &
    S & R & $\Delta$ \\
    \hline
    none &
    81.13 & 81.13 & 0.00 &
    76.28 & 76.28 & 0.00 &
    81.29 & 81.29 & 0.00 &
    84.01 & 84.01 & 0.00 \\
    \hline
    gaussian noise &
    21.12 & 26.81 & \textbf{5.70} &
    47.89 & 51.07 & \textbf{3.18} &
    65.91 & 66.34 & \textbf{0.43} &
    20.87 & 28.18 & \textbf{7.31} \\
    \hline
    shot noise &
    29.96 & 36.34 & \textbf{6.38} &
    55.69 & 58.24 & \textbf{2.55} &
    70.39 & 70.72 & \textbf{0.33} &
    31.12 & 39.37 & \textbf{8.25} \\
    \hline
    impulse noise &
    19.21 & 26.09 & \textbf{6.88} &
    59.68 & 59.09 & -0.59 &
    79.72 & 76.08 & -3.64 &
    37.18 & 37.01 & -0.17 \\
    \hline
    defocus blur &
    64.44 & 64.44 & 0.00 &
    73.42 & 73.42 & 0.00 &
    78.23 & 78.23 & 0.00 &
    69.92 & 69.92 & 0.00 \\
    \hline
    glass blur &
    20.68 & 26.61 & \textbf{5.93} &
    54.08 & 56.09 & \textbf{2.00} &
    58.82 & 60.48 & \textbf{1.66} &
    31.07 & 37.62 & \textbf{6.55} \\
    \hline
    motion blur &
    60.25 & 60.25 & 0.00 &
    70.46 & 70.46 & 0.00 &
    74.73 & 74.73 & 0.00 &
    66.14 & 66.14 & 0.00\\
    \hline
    zoom blur &
    59.58 & 59.58 & 0.00 &
    71.83 & 71.83 & 0.00 &
    76.71 & 76.71 & 0.00 &
    65.10 & 65.10 & 0.00\\
    \hline
    snow &
    59.07 & 59.07 & 0.00 &
    65.83 & 65.83 & 0.00 &
    71.83 & 71.83 & 0.00 &
    70.99 & 70.99 & 0.00 \\
    \hline
    frost &
    54.35 & 54.35 & 0.00 &
    63.69 & 63.69 & 0.00 &
    71.00 & 71.00 & 0.00 &
    65.21 & 65.21 & 0.00 \\
    \hline
    fog &
    71.25 & 71.25 & 0.00 &
    66.59 & 66.59 & 0.00 &
    72.84 & 72.84 & 0.00 &
    77.21 & 77.21 & 0.00 \\
    \hline
    brightness &
    76.94 & 76.94 & 0.00 &
    73.50 & 73.50 & 0.00 &
    78.32 & 78.32 & 0.00 &
    80.13 & 80.13 & 0.00 \\
    \hline
    contrast &
    62.97 & 62.97 & 0.00 &
    65.25 & 65.25 & 0.00 &
    68.03 & 68.03 & 0.00 &
    72.88 & 72.88 & 0.00 \\
    \hline
    elastic &
    64.23 & 64.23 & 0.00 &
    68.09 & 68.09 & 0.00 &
    73.43 & 73.43 & 0.00 &
    68.58 & 68.58 & 0.00 \\
    \hline
    pixelate &
    54.28 & 54.28 & 0.00 &
    63.54 & 63.54 & 0.00 &
    70.46 & 70.46 & 0.00 &
    52.34 & 52.34 & 0.00 \\
    \hline
    jpeg &
    50.40 & 50.40 & 0.00 &
    62.11 & 62.11 & 0.00 &
    69.24 & 69.24 & 0.00 &
    51.71 & 51.71 & 0.00 \\
    \hline
    speckle noise &
    31.58 & 35.76 & \textbf{4.18} &
    58.11 & 59.33 & \textbf{1.23} &
    71.67 & 71.57 & -0.09 &
    33.96 & 38.94 & \textbf{4.98} \\
    \hline
    gaussian blur &
    54.11 & 54.11 & 0.00 &
    71.48 & 71.48 & 0.00 &
    76.74 & 76.74 & 0.00 &
    61.11 & 61.11 & 0.00 \\
    \hline
    spatter &
    61.23 & 62.23 & \textbf{1.00} &
    72.28 & 71.25 & -1.02 &
    78.11 & 77.44 & -0.67 &
    79.73 & 78.90 &  -0.83 \\
    \hline
    saturate &
    68.82 & 68.88 & \textbf{0.06} &
    64.52 & 64.77 & \textbf{0.25} &
    69.83 & 70.21 & \textbf{0.39} &
    72.93 & 72.99 & \textbf{0.05} \\
    \hline
\end{tabular}
\end{center}
\end{table*}


%% file: supp_sections/clustering.tex
Here we examine the how our choice of surrogate image corruptions impacted \ours{} performance. Figure~\ref{fig:kmeans} shows the projected state representations for 100 samples of our best performing shifts and each surrogate shift, sorted into $K=2$ clusters by $K$-means clustering. The value of $K$ was selected from the range $[1,9]$ via the elbow method. We see that samples from ImageNet-C noise-type shifts are in the same cluster as those from our uniform noise shift. Likewise, brightness, contrast, and saturate are clustered with our $\gamma$ correction shift. Thus, \ours{} generalizes well from the surrogates to these ImageNet-C shifts.

\begin{figure}[h]
\centering\includegraphics[width=0.6\linewidth]{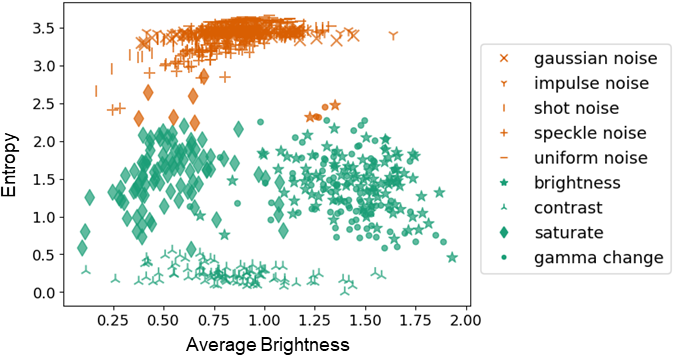}
    \caption{\small $K$-means clustering of sampled state representations for best-performing shifts from ImageNet-C and our surrogates (projected onto two dimensions). The surrogate shifts are clustered with the appropriate ImageNet-C shifts. \vspace{-2mm}}
    \label{fig:kmeans}
\end{figure}

%% file: supp_sections/limitations.tex
In this supplementary section, we explore the limitations of \ours{} in greater detail.

\textbf{Action Library.} One important limitation of \ours{} is the technique's reliance on an appropriate selection of the action library. Since the policy network selects corrective actions from this library, by design our method can only recover from distribution shifts targeted by the action library. For example, for the ImageNet-C benchmark, our action library caters to noise-related and $\gamma$-level-related distribution shifts. The result is that \ours{} overlwhelmingly performs better on these types of ImageNet-C shifts, in contrast to blur-related, weather-related, and corruption-related shifts. Furthermore, it is not feasible to simply select a large, comprehensive action library, as this prohibitively increases the search space that must be traversed by the reinforcement learning algorithm.

\textbf{Surrogate Corruptions.} Much like the action library, the selection of surrogate corruptions can pose a challenge in implementing \ours{}. Our method hinges on the reinforcement learning agent interpolating knowledge about the surrogate corruptions to apply to the test-time corruptions (e.g., ImageNet-C or CIFAR-100-C). If there is large deviation between the set of surrogate corruptions and the set of test-time corruptions, the successful application of knowledge to the test data is less likely. However, this is a natural consequence of any data-driven solution, and the flexibility lent by using an offline approach mitigates this challenge.

\textbf{Responsiveness.} To estimate the Wasserstein distance in our reinforcement learning reward function, our method requires a set of data samples already subject to distribution shift. In high dimensions, the sample complexity of this estimate grows~\citep{ramdas2017wasserstein}. Even with dimensionality reduction techniques, a large sample size may be required, as the reduction in dimensions cannot be so large as to lose essential information in the data. In these cases, the responsiveness of \ours{} may be limited in speed, as \ours{} must wait for a large number of samples to begin adapting and recovering from distribution shift. We note that a quick response time is not an essential requirement in our setting, as we assume that when distribution shift arises, it persists for a certain duration of time. This is not a stringent assumption for the naturally occurring distribution shifts we consider in this work.

\textbf{Theoretical Guarantees.} Since our method views the end model we wish to adapt as a black box, we can make no theoretical guarantees on how \ours{} affects the class-wise performance of the model. However, we provide theoretical guarantees (see Theorem~\ref{thm}) that selecting transformations that minimize the distance between the training distribution and the corrupted distribution in turn minimizes the overall loss in performance on the corrupted data after correction.